\documentclass[letterpaper]{article} % DO NOT CHANGE THIS
\usepackage{aaai25}  % DO NOT CHANGE THIS
\usepackage{times}  % DO NOT CHANGE THIS
\usepackage{helvet}  % DO NOT CHANGE THIS
\usepackage{courier}  % DO NOT CHANGE THIS
\usepackage[hyphens]{url}  % DO NOT CHANGE THIS
\usepackage{graphicx} % DO NOT CHANGE THIS
\usepackage{natbib}  % DO NOT CHANGE THIS AND DO NOT ADD ANY OPTIONS TO IT
\usepackage{caption} % DO NOT CHANGE THIS AND DO NOT ADD ANY OPTIONS TO IT
\usepackage{booktabs}
\usepackage[switch]{lineno}

\usepackage{amsmath, amsthm, amsxtra, amsfonts, amssymb, amstext, mathtools}%,
\usepackage{nicefrac}
\usepackage{xspace}
\usepackage{xcolor}
\usepackage[ruled,vlined,linesnumbered]{algorithm2e}\SetArgSty{upshape}
\usepackage[nameinlink]{cleveref}

\usepackage[shortcuts]{extdash} % Should always come last!

\newtheorem{theorem}{Theorem}
\newtheorem{lemma}[theorem]{Lemma}
\newtheorem{corollary}[theorem]{Corollary}

\newcommand{\nsga}{NSGA\=/II\xspace}
\newcommand{\nsgathree}{NSGA\=/III\xspace}

\newcommand{\LO}{\textsc{Leading\-Ones}\xspace}
\newcommand{\leadingones}{\LO}

\newcommand{\jump}{\textsc{Jump}\xspace}
\newcommand{\omm}{\textsc{OMM}\xspace}
\newcommand{\cocz}{\textsc{COCZ}\xspace}

\newcommand{\lotz}{\textsc{LOTZ}\xspace}
\newcommand{\ojzj}{\textsc{OJZJ}\xspace}
\newcommand{\ommLong}{\textsc{OneMinMax}\xspace}
\newcommand{\lotzLong}{\textsc{LeadingOnesTrailingZeros}\xspace}
\newcommand{\ojzjLong}{\textsc{OneJumpZeroJump}\xspace}

\newcommand{\R}{\ensuremath{\mathbb{R}}}

\newcommand{\N}{\ensuremath{\mathbb{N}}}
\newcommand{\Z}{\ensuremath{\mathbb{Z}}}

\DeclareMathOperator{\Bin}{Bin}

\newcommand{\cdis}{\mathrm{cDis}}

\newcommand{\eps}{\varepsilon}

\newcommand{\card}[1]{\lvert#1\rvert}

\let\originalleft\left
\let\originalright\right
\renewcommand{\left}{\mathopen{}\mathclose\bgroup\originalleft}
\renewcommand{\right}{\aftergroup\egroup\originalright}

\title{Speeding Up the NSGA-II With a Simple Tie-Breaking Rule}

\author{
  Benjamin Doerr\textsuperscript{\rm 1},
  Tudor Ivan\textsuperscript{\rm 2},
  Martin~S. Krejca\textsuperscript{\rm 1}
}
\affiliations{
  \textsuperscript{\rm 1}Laboratoire d'Informatique (LIX), CNRS, École Polytechnique, Institut Polytechnique de Paris\\
  \textsuperscript{\rm 2}École Polytechnique, Institut Polytechnique de Paris\\
  \{first-name.last-name\}@polytechnique.edu
}

\begin{document}

\maketitle

\sloppy{
\begin{abstract}
  The non-dominated sorting genetic algorithm~II (\nsga) is the most popular multi-objective optimization heuristic. Recent mathematical runtime analyses have detected two shortcomings in discrete search spaces, namely, that the NSGA-II has difficulties with more than two objectives and that it is very sensitive to the choice of the population size.
  To overcome these difficulties, we analyze a simple tie-breaking rule in the selection of the next population. Similar rules have been proposed before, but have found only little acceptance. We prove the effectiveness of our tie-breaking rule via mathematical runtime analyses on the classic \ommLong, \lotzLong, and \ojzjLong benchmarks. We prove that this modified NSGA-II can optimize the three benchmarks efficiently also for many objectives, in contrast to the exponential lower runtime bound previously shown for \ommLong with three or more objectives. For the bi-objective problems, we show runtime guarantees that do not increase when moderately increasing the population size over the minimum admissible size. For example, for the \ojzjLong problem with representation length~$n$ and gap parameter~$k$, we show a runtime guarantee of $O(\max\{n^{k+1},Nn\})$ function evaluations when the population size is at least four times the size of the Pareto front. For population sizes larger than the minimal choice $N = \Theta(n)$, this result improves considerably over the $\Theta(Nn^k)$ runtime of the classic \nsga.
\end{abstract}

\section{Introduction}

Many real-world optimization tasks face several, usually conflicting objectives. One of the most successful approaches to such \emph{multi-objective} optimization problems are evolutionary algorithms (EAs) \cite{CoelloLV07,ZhouQLZSZ11}, with the \nsga~\cite{DebPAM02} standing out as the by far dominant algorithm in practice (over $50\,000$ citations on Google Scholar).

Recent mathematical works have shown many positive results for the \nsga (see \emph{Previous Works}) but have also exhibited two difficulties. (1)~For more than two objectives, the crowding distance as selection criterion seems to have some shortcomings. This was formally proven for the optimization of the simple \ommLong benchmark, where an exponential lower bound on the runtime was shown for three and more objectives when the population size is linear in the size of the Pareto front~\cite{ZhengD24many}. The proof of this result suggests that similar problems exist for many optimization problems with three or more objectives. (2)~All runtime guarantees proven for the \nsga increase linearly with the population size. For some settings, even matching lower bounds proving this effect were proven~\cite{DoerrQ23LB}. This behavior is very different from many single-objective EAs, where a moderate increase of the population size increases the cost of an iteration but at the same time reduces the number of iterations in a way that the total runtime (number of function evaluations) is at most little affected.

The  reason for this undesired behavior of the \nsga, as the proofs by \citet{DoerrQ23LB} reveal, is that the population of the \nsga is typically not evenly distributed on the known part of the Pareto front.
Instead, the distribution is heavily skewed toward the inner region of the Pareto front.

Inspired by this observation, we propose to add a simple tie-breaking criterion to the selection of the next population. In the classic \nsga, selection is done according to non-dominated sorting, ties are broken according to the crowding distance, and remaining ties are broken randomly. We add the number of individuals having a certain objective value as the third criterion, and break only the remaining ties randomly. We note that a similar idea was suggested already by \citet{FortinP13} and supported by empirical results, but has not made it into the typical use of the \nsga.

Our tie-breaker solves both problems. We rigorously study the runtime for the three most established benchmarks in the theory of multi-objective EAs (MOEAs): \ommLong, \lotzLong, and \ojzjLong.
For all defined versions with at least three objectives and for constant gap parameter~$k$, we show that our \nsga with a population size exceeding the Pareto front size only by a lower-order term efficiently solves the problem when the Pareto front size is polynomial (\Cref{thm:manyo}).

For the bi-objective versions of these problems, we prove runtime guarantees showing that for a certain range of the population size, the runtime is (asymptotically) not affected by this parameter. The size of this range depends on the difficulty of the problem.
For the difficult \ojzjLong benchmark with problem size~$n$ and gap parameter~$k$, our runtime guarantee is $O(n^{k+1})$ for all population sizes $N$ between $4(n-2k+3)$ and $O(n^k)$, and it is $O(Nn)$ for $N = \Omega(n^k)$. Compared to the runtime of $\Theta(Nn^k)$ proven by \citet{DoerrQ23tec}, this is a noteworthy speed-up for larger population sizes. From a practical perspective, this result indicates that our tie-breaking rule significantly reduces the need for a careful optimization of the algorithm parameter~$N$.
We support this claim empirically, showing that this speed-up is already noticeable when the chosen and optimal population size deviate only by a constant factor.

Overall, this work shows that adding the simple tie-breaker of preferring individuals with rarer objective values can lead to considerable performance gains and greatly reduce the need of determining an optimal population size.

All proofs and some additional details are in the appendix.

\section{Previous Works}
\label{sec:previousWork}

This being a theoretical work, for space reasons, we refer to the surveys by Coello et~al.~\shortcite{CoelloLV07} and by \citet{ZhouQLZSZ11} for the success of MOEAs in practical applications.

The mathematical analysis of randomized search heuristics has supported for a long time the development of these algorithms~\cite{NeumannW10,AugerD11,Jansen13,ZhouYQ19,DoerrN20}, including MOEAs. The runtime analysis for MOEAs was started in~\cite{LaumannsTZWD02,Giel03,Thierens03} with very simplistic algorithms like the \emph{simple evolutionary multi-objective optimizer (SEMO)} or the \emph{global SEMO}. Due to the complex population dynamics, it took many years until more prominent MOEAs could be analyzed, such as the $(\mu+1)$ SIBEA~\cite{BrockhoffFN08}, MOEA/D~\cite{LiZZZ16}, \nsga~\cite{ZhengLD22}, \nsgathree~\cite{WiethegerD23}, or SMS-EMOA~\cite{BianZLQ23}.

The analysis of the \nsga, the by far dominant algorithm in practice, had a significant impact on the field and was quickly followed up by other runtime analyses for this algorithm. The vast majority of these prove runtime guarantees for bi-objective problems. Some are comparable to those previously shown for the (G)SEMO~\cite{BianQ22,DoerrQ23tec,DangOSS23aaai,CerfDHKW23,DengZLLD24}, others explore new phenomena like approximation properties~\cite{ZhengD24approx}, better robustness to noise~\cite{DangOSS23gecco}, or new ways how crossover~\cite{DoerrQ23crossover} or archives~\cite{BianRLQ24} can be advantageous. All these works, except for~\cite{BianRLQ24}, require that the population size is at least a constant factor larger than the size of the Pareto front. Increasing the population size further leads to a proportional increase of the runtime guarantee. Lower bounds for the \ommLong and \ojzjLong benchmarks show that this increase is real, i.e., the runtime is roughly proportional to the population size. This runtime behavior is very different from what is known from single-objective optimization, where results like~\cite{JansenJW05,Witt06,DoerrK15} show that often there is a regime where an increase of the population size does not lead to an increase of the runtime. Hence, for the \nsga, the choice of the population size is much more critical than for many single-objective algorithms (or the (G)SEMO with its flexible population size). The results most relevant for our work are the upper and lower bounds for the bi-objective \ommLong, \lotzLong, and \ojzjLong problems. We describe these in the sections where they are relevant.

So far, there is only one runtime analysis of the \nsga on a problem with more than two objectives \cite{ZhengD24many}.
It shows that the \nsga takes at least exponential time to optimize \ommLong in at least three objectives, for any population size at most a constant factor larger than the Pareto front (see also further below). This result was recently extended to \lotzLong \cite{DoerrKK24arxiv}.

\section{Preliminaries}
\label{sec:preliminaries}

The natural numbers~$\N$ include~$0$.
For $m, n \in \N$, we define $[m..n] \coloneqq [m,n] \bigcap \N$ as well as $[n] \coloneqq [1..n]$.

Given $m, n \in \N$, an \emph{$m$-objective function}~$f$ is a tuple $(f_j)_{j \in [m]}$ where, for all $j \in [m]$, it holds that $f_j\colon \{0, 1\}^n \to \R$.
Given an $m$-objective function, we implicitly assume that we are also given~$n$.
We call each $x \in \{0, 1\}^n$ an \emph{individual} and $f(x) \coloneqq (f_j(x))_{j \in [m]}$ the \emph{objective value of~$x$}.
For each $i \in [n]$, we denote the $i$-th component of~$x$ by~$x_i$.
We denote the number of~$1$s of~$x$ by~$|x|_1$, and its number of~$0$s by~$|x|_0$.

We consider the \emph{maximization} of $m$-objective functions.
The objective values of an $m$-objective function~$f$ induce a weak partial order on the individuals, denoted by~$\succeq$.
For $x, y \in \{0, 1\}^n$, we say that \emph{$x$ weakly dominates~$y$} (written $x \succeq y$) if and only if for all $j \in [m]$ holds that $f_j(x) \ge f_j(y)$.
If one of these inequalities is strict, we say that~$x$ \emph{strictly} dominates~$y$ (written $x \succ y$).
We say that~$x$ is \emph{Pareto-optimal} if and only if~$x$ is not strictly dominated by any individual.
We call the set of objective values of all Pareto-optimal individuals the \emph{Pareto front} of~$f$, and we call the set of all Pareto-optimal individuals the \emph{Pareto set} of~$f$.

\subsection{The \nsga}
\label{sec:preliminaries:nsga}

The non-dominated sorting genetic algorithm~II (\nsga, \Cref{alg:NSGA}) is the most popular heuristic for multi-objective optimization.
It optimizes a given $m$-objective function iteratively, maintaining a multi-set (a \emph{population}) of individuals of a given size $N \in \N_{\geq 1}$.
This population is initialized with individuals chosen uniformly at random.

In each iteration, the \nsga generates an additional, new \emph{offspring} population of~$N$ individuals as we detail below.
Out of the~$2N$ individuals from the combined population of current and offspring population, the algorithm selects~$N$ individuals as the new population for the next iteration.
To this end, the \nsga utilizes two characteristics defined over individuals, which we also both detail below: \emph{non-dominated ranks} and \emph{crowding distance}.
The~$2N$ individuals are sorted lexicographically by first minimizing the rank and then by maximizing the crowding distance.
Ties are broken uniformly at random (u.a.r.).
The first~$N$ individuals from the sorted population are kept for the next iteration.

\Cref{alg:NSGA} showcases the \nsga in a way suited toward our modification described in the following section.

\paragraph{Runtime.}
The runtime of an algorithm optimizing a function~$f$ is the (random) number of evaluations of~$f$ until the objective values of the population cover the Pareto front of~$f$.
We assume that the objective value of each individual is evaluated exactly once when it is created.
Thus, the \nsga uses~$N$ function evaluations when initializing the population, and~$N$ function evaluations per iteration for the offspring population.
As the number of function evaluations is essentially~$N$ times the number of iterations, we use the term \emph{runtime} interchangeably for both quantities and always specify which one we mean.
We state our runtime results in big-O notation with asymptotics in the problem size~$n$.

\paragraph*{Offspring generation.}
Given a \emph{parent} population~$P$, the \nsga creates~$N$ \emph{offspring} by repeating the following \emph{standard bit mutation}~$N$ times:
Choose an individual $x \in P \subseteq \{0, 1\}^n$ uniformly at random, and create a copy~$y$ of~$x$ where one flips each component of~$x$ independently with probability~$1/n$.
That is, for all $i \in [n]$, we have $y_i = 1 - x_i$ with independent probability~$1/n$, and $y_i = x_i$ otherwise.

\paragraph*{Non-dominated ranks.}
Given a population~$R$, the \emph{rank} of each $x \in R$ is defined inductively and roughly represents how many layers in~$R$ dominate~$x$.
Individuals in~$R$ that are not strictly dominated by any individual in~$R$ receive rank~$1$.
The population of all such individuals is denoted as $F_1 \coloneqq \{x \in R \mid \forall y \in R\colon x \nprec y\}$.
The population of all individuals of rank $j \in \N_{\geq 2}$ is the population of all individuals in~$R$, after removing all individuals of ranks~$1$ to~$j - 1$, that are not strictly dominated.
That is, $F_j \coloneqq \{x \in R \setminus \bigcup_{k \in [j - 1]} F_k \mid \forall y \in R \setminus \bigcup_{k \in [j - 1]}\colon x \nprec y\}$.
If there are in total $r \in \N_{\geq 1}$ different ranks of~$R$, then $(F_j)_{j \in [r]}$ is a partition of~$R$.

Given such a partition of~$R$, we say that a rank $j^* \in [r]$ is \emph{critical} if and only if~$j^*$ is the minimal index such that the population of all individuals up to rank~$j^*$ is at least~$N$.
That is, $|\bigcup_{j \in [j^* - 1]} F_j| < N$ and $|\bigcup_{j \in [j^*]} F_j| \geq N$.
Individuals with a rank strictly smaller than~$j^*$ are definitely selected for the next iteration, and individuals with a strictly larger rank than~$j^*$ are definitely not selected.
Among individuals with a rank of exactly~$j^*$, some individuals may be selected and some not, thus requiring further means to make a choice.

\paragraph*{Crowding distance.}
Given a population~$F$, the \emph{crowding distance} (CD) of each $x \in F$, denoted by $\cdis(x)$, is the sum of the CD of~$x$ \emph{per objective}.
The CD of~$x$ for objective $j \in [m]$ is as follows:
Let $a = |F|$, and let $(S_{j.i})_{i \in [a]}$ denote~$F$ sorted by increasing value in objective~$j$, breaking ties arbitrarily.
The CD of~$S_{j.1}$ and of~$S_{j.a}$ for objective~$j$ is infinity.
For all $i \in [2 .. a - 1]$, the CD of~$S_{j.i}$ for objective~$j$ is $\bigl(f(S_{j.i + 1}) - f(S_{j.i - 1})\bigr)/\bigl(f(S_{j.N}) - f(S_{j.1})\bigr)$.

Let $d \in [|F|]$, and let $(C_c)_{c \in [k]}$ be a partition of~$F$ such that for all $c \in [k]$, all individuals in~$C_c$ have the same CD and that for all $c_1, c_2 \in [k]$ with $c_1 < c_2$, the CD of~$C_{c_1}$ is strictly larger than that of~$C_{c_2}$.
We say $c^* \in [k]$ is the \emph{critical CD index of $(C_c)_{c \in [k]}$ with respect to~$d$} if and only if $|\bigcup_{c \in [c^* - 1]} C_c| < d$ and $|\bigcup_{c \in [c^*]} C_c| \geq d$.
When selecting~$d$ individuals from~$F$, individuals with a CD in a population of index less than~$c^*$ are definitely selected, and those with a CD of a population with a strictly larger index are not.
From~$C_{c^*}$, some individuals may be selected and some not.

\begin{algorithm}[t]
  \caption{\label{alg:NSGA}
    The (classic) non-dominated sorting genetic algorithm~II (\nsga) with population size $N\in\N_{\geq 1}$, optimizing an $m$-objective function.
  }
  $P_0 \gets$ population of~$N$ individuals, each u.a.r.\;
  $t \gets 0$\;
  \While{termination criterion not met}{
    $Q_t \gets$ offspring population of~$P_t$\;
    $R_t \gets P_t \cup Q_t$\;
    $(F_j)_{j \in [r]} \gets$ partition of~$R_t$ w.r.t. non-dom. ranks\;
    $j^* \gets$ critical rank of $(F_j)_{j \in [r]}$\;
    $(C_c)_{c \in [k]} \gets$ partition of~$F_{j^*}$ w.r.t. crowd. dist.\;
    $c^* \gets$ critical crowding distance index of $(C_c)_{c \in [k]}$ w.r.t. $N - |\bigcup_{j \in [j^* - 1]} F_j|$\;
    $s \gets N - |\bigcup_{j \in [j^* - 1]} F_j \cup \bigcup_{c \in [c^* - 1]} C_c|$\;
    $W \gets$ sub-pop. of~$C_{c^*}$ of cardinality~$s$, u.a.r.\;\label{alg:NSGA_tieBreaker}
    $P_{t + 1} \gets \bigcup_{j \in [j^* - 1]} F_j \cup \bigcup_{c \in [c^* - 1]} C_c \cup W$\;
    $t \gets t + 1$\;
  }
\end{algorithm}

\subsection{Benchmarks}  \label{sec:preliminaries:benchmarks}

We consider the three most common functions for the theoretical analysis of multi-objective search heuristics (see \emph{Previous Works}). We define here their bi-objective versions, which are the most common ones, and build on these definitions later when defining the many-objective analogs.

The \ommLong (\omm) benchmark~\cite{GielL10} returns the number of~$0$s and~$1$s of each individual, formally, $\omm\colon x \mapsto (|x|_0, |x|_1)$.
Each individual is Pareto-optimal.
The Pareto front is $\{(i, n - i) \mid i \in [0 .. n]\}$.

The \ojzjLong (\ojzj) benchmark~\cite{DoerrZ21aaai} extends the classic \jump benchmark \cite{DrosteJW02} to several objectives. It is defined similarly as \omm but has an additional parameter $k \in [2 .. n]$.
It is effectively identical to \omm for all individuals whose number of~$1$s is between~$k$ and $n- k$ or is~$0$ or~$n$.
The objective values of these individuals constitute the Pareto front of the function.
For all other individuals, the objective value is strictly worse.
Hence, once an algorithm finds solutions with~$k$ or $n - k$ $1$s, it needs to change at least~$k$ positions at once in a solution in order to expand the Pareto front, which is usually a hard task.
Formally, for all $i \in \{0, 1\}$ and all $x \in \{0, 1\}^n$, let
\begin{equation*}
  J^{(i)}(x) =
  \begin{cases}
    k + |x|_i & \textrm{if } |x|_i \in [0 .. n - k] \cup \{n\}, \\
    n - |x|_i & \textrm{else.}
  \end{cases}
\end{equation*}
Then $\ojzj(x) = (J^{(1)}(x), J^{(0)}(x))$, with the Pareto front $\{(i, n + 2k - i) \mid i \in [2k .. n] \cup \{k, n + k\}\}$.

The \lotzLong (\lotz) benchmark~\cite{LaumannsTZ04}, a multi-objective version of \leadingones \cite{Rudolph97}, returns the length of the longest prefix of~$1$s and suffix of~$0$s, formally
\begin{equation*}
  x \mapsto \bigl(\textstyle\sum\nolimits_{i \in [n]} \prod\nolimits_{j \in [i]} x_j, \sum\nolimits_{i \in [n]} \prod\nolimits_{j \in [i .. n]} (1 - x_j)\bigr) .
\end{equation*}
The Pareto front is the same as that of \omm, but the Pareto set is $\{1^i 0^{n - i} \mid i \in [0 .. n]\}$.

\section{Improved Tie-Breaking for the \nsga}
\label{sec:improvedTieBreaking}

The \nsga selects individuals elaborately; first via the non-dominated ranks, second via the crowding distance, and last uniformly at random.
While the final tie-breaker seems reasonable, it neglects the structure of the population.

In more detail, since the uniform selection is performed over the subpopulation~$C_{c^*}$ from \cref{alg:NSGA_tieBreaker} in \Cref{alg:NSGA}, any imbalances with respect to different objective values are carried over in expectation to the selected population~$W$.
That is, if most of the individuals from~$C_{c^*}$ have objective value $v_1 \in \R^m$ and only very few have objective value $v_2 \in \R^m$, then it is more likely for individuals with objective value~$v_1$ to be selected although individuals with objective value~$v_2$ might also have interesting properties.
To circumvent this problem, we propose to select the individuals from~$C_{c^*}$ as evenly as possible from all the different objective values.

\paragraph{Balanced tie-breaking.}
We replace \cref{alg:NSGA_tieBreaker} in \Cref{alg:NSGA} with the following procedure, using the same notation as in the pseudo code:
Partition~$C_{c^*}$ with respect to its objective values into $(C'_{c})_{c \in [a]}$, assuming~$a$ different objective values in~$C_{c^*}$.
That is, for $U \coloneqq \{f(x) \mid x \in C_{c^*}\}$ (being a set without duplicates) and for each $u \in U$, there is exactly one $c \in [a]$ with $C'_{c} = \{x \in C_{c^*} \mid f(x) = u\}$ (where~$C'_c$ is a multi-set).
For each $c \in [a]$, select $\min(|C'_c|, \lfloor s/a \rfloor)$ individuals uniformly at random from~$C'_c$, calling this selected population $\widetilde{C}_c$.
That is, $\widetilde{C}_c \subseteq C'_c$ with $|\widetilde{C}_c| = \min(|C'_c|, \lfloor s/a \rfloor)$, chosen uniformly at random among all sub-multi-sets of~$C'_c$ of cardinality $\min(|C'_c|, \lfloor s/a \rfloor)$.
Add all individuals in $\bigcup_{c \in [a]} \widetilde{C}_c$ to~$W$.
If this does not select sufficiently many individuals, that is, if $|\bigcup_{c \in [a]} \widetilde{C}_c| < s$, then select the missing number of individuals uniformly at random from the remaining population, that is, from $C_{c^*} \setminus \bigcup_{c \in [a]} \widetilde{C}_c$.

Since this tie-breaking aims at balancing the amount of individuals per objective value during the third tie-breaker, we call the modified algorithm the \emph{balanced} \nsga.

\textbf{Additional cost.}
Based on our experiments in the empirical section, we observed that balanced tie-breaking is slower than random tie-breaking in terms of wall clock time by a factor of around~$10$ on average.
However, the total time spent on balanced tie-breaking is still, on average, only~$15$\,\% of the total time spent on non-dominated sorting, which is always required during selection.
Moreover, as we detail in our empirical evaluation, the overall number of function evaluations (and thus wall clock time) of the balanced \nsga is typically far faster than that of the classic \nsga.

\subsection{Properties of the Balanced \nsga}

For the classic \nsga, if the population size~$N$ is large enough w.r.t. the Pareto front of the objective function, no value on the Pareto front is lost.
We prove that this same useful property also holds for the balanced \nsga.

The following lemma proves an upper bound on the number of individuals with positive crowding distance among those with critical rank.
The lemma is adapted from an argument in the proof of Lemma~$1$ by Zheng et~al.~\shortcite{ZhengLD22}.

\begin{lemma}
  \label{lem:max2mpositives}
  Consider the balanced \nsga optimizing an $m$-objective function~$f$. For each iteration $t \in \N$ we have that for each objective value in the critical rank $A \in f(F_{i^*})$ there exist at most $2m$ individuals $x \in F_{i^*} \subseteq R_t$ with $f(x) = A$ such that $\cdis(x) > 0$.
\end{lemma}

\Cref{lem:max2mpositives} yields that in the balanced \nsga there is always a fair number of individuals with critical rank.
\begin{lemma}
  \label{lem:numberOfElementsBound}
  Consider the balanced \nsga optimizing an $m$-objective function.
  Assume that at some iteration $t \in \N$ we select $C \in \N$ individuals from the critical rank $F_{i^*} \subseteq R_t$ with size of the objective values set $\card{f(F_{i^*})} = S$.
  Then, for any $A \in f(F_{i^*})$ we keep at least $\min \left( \max(\lfloor \frac{C}{S} \rfloor - 2m, 0), \card{ \{x \in F_{i^*} \mid f(x) = A \}} \right) $ individuals with $f(x) = A$ in $P_{t + 1}$.
\end{lemma}

\section{The Balanced \nsga Is Efficient For Three or More Objectives}
\label{sec:manyObjective}

We analyze the performance of the balanced \nsga on \omm, \lotz, and \ojzj with three or more objectives.
We show that the balanced \nsga optimizes these benchmarks in polynomial time when the number~$m$ of objectives (and the gap parameter of \ojzj) is constant (\Cref{thm:manyo}).
This result stands in strong contrast to the performance of the classic \nsga.
Recently, \citet{ZhengD24many} proved that the classic \nsga with any population size linear in the Pareto front size cannot optimize \omm with $m \ge 3$ objectives faster than in time $\exp(\Omega(n^{\lceil m/2 \rceil}))$. Their proofs suggest that the classic \nsga has similar difficulties on many other many-objective problems (i.e., three or more objectives), including \lotz and \ojzj.

We briefly state the definitions of the $m$-objective versions of \omm, \lotz, and \ojzj from~\cite{ZhengD24many,LaumannsTZ04,ZhengD24} (precise definitions in the appendix). All three lift the definition of the two-objective problem to an even number $m$ of objectives by splitting the bit string into $m/2$ equal-length segments and then taking as $2i-1$-st and $2i$-th objective the original function applied to the $i$-th block.

We review the (few) main existing runtime result for these benchmarks. In the first mathematical runtime analysis for a many-objective problem, Laumanns et~al.~\shortcite{LaumannsTZ04} showed that the SEMO algorithm optimizes the \cocz and \lotz problems in an expected number of $O(n^{m+1})$ function evaluations. \cocz is similar to \omm, so it is quite clear that the relevant part of their proof also applies to \omm, giving again an $O(n^{m+1})$ bound. Also, it is easy to see that their analysis can be extended to the GSEMO, giving the same runtime guarantees. The bounds for \cocz were  improved slightly to $O(n^m)$, and $O(n^3 \log n)$ for $m = 4$, by Bian et~al.~\shortcite{BianQT18ijcaigeneral}. \citet{HuangZLL21} analyzed how the MOEA/D optimizes \cocz and \lotz. We skip the details since the MOEA/D is very different from all other algorithms discussed in this work. \citet{WiethegerD23} proved a runtime guarantee of $O(Nn\log n)$ for the \nsgathree optimizing the $3$-objective \omm problem when the population size is at least the size of the Pareto front.
The only many-objective results for $\ojzj_k$ are an $O(M^2 n^k)$ bound for the GSEMO and an $O(\mu M n^k)$ bound for the SMS-EMOA with population size $\mu \ge M$, where~$M$ is the size of the Pareto front, see \cite{ZhengD24}.
Note that as the runtimes of many-objective problems are not too well understood, and in the absence of any reasonable lower bound, there is a high risk that the results above are far from tight.

Before we state our main result of this section, we note that our main goal is to show the drastic difference to the behavior of the classic \nsga exhibited by \citet{ZhengD24many}.
We do not optimize our runtime estimates with more elaborate methods but are content with polynomial-time bounds for constant~$m$ and~$k$ and population sizes linear in the Pareto front size. Near-tight bounds for many-objective evolutionary optimization of our benchmarks were recently proven in~\cite{WiethegerD24}.
For the same reason, we also do not prove bounds that do not show an increase of the runtime with growing population size in certain ranges (as we do for two objectives), though clearly this would be possible with similar arguments.

\begin{theorem}\label{thm:manyo}
  Let $m \in \N$ be even and $m'=m/2$. Assume that $n'=n/m' \in \Z$. Let $k \in [2..n'/2]$. Consider the \omm, \lotz, or $\ojzj_k$ problem. Denote by $M$ the size of the Pareto front and by $S$ the size of a largest set of pairwise incomparable solutions. Assume that we optimize these problems via the balanced \nsga with population size $N \ge {S + 2m(n'+1)} = S + 4n + 2m$. Then we have the following bounds for the expected runtime:
  \begin{enumerate}
    \item For \omm, it is at most $2enM$ iterations.
    \item For \lotz, it is at most $2enM + 2en^2$ iterations.
    \item For $\ojzj_k$, it is at most $2en^k M + 2ekm'n$ iterations.
  \end{enumerate}
\end{theorem}

As many runtime analyses, our bounds depend on the size $S$ of the largest incomparable set of solutions the problem admits. For this, the following bounds are known or can easily be found: For \omm, we have $S = M = (n'+1)^{m/2}$, for \lotz, we have $S \le (n'+1)^{m-1}$ \cite{OprisDNS24}, and for \ojzj, we have $S \le (n'+1)^{m/2}$.

The reason for the drastically different behavior of the classic and the balanced \nsga is that the former can lose Pareto optimal solution values with any population size that is linear in the size of the Pareto front. In contrast, for the balanced \nsga often a population size exceeding the Pareto front size only by a lower-order term suffices to prevent such a loss of objective values. This follows easily from arguments similar to those used to prove Lemma~\ref{lem:numberOfElementsBound}. For the convenience of this and possible future works, we formulate and prove this crucial statement as a separate lemma.

\begin{lemma}\label{lem:survival}
  Consider the balanced \nsga with population size $N$ optimizing some $m$-objective optimization problem. Assume that $S$ is an upper bound on the size of any set of pair-wise incomparable solutions. Assume that $U$ is an upper bound on the number of individuals with positive crowding distance in a set of solutions such that any two are incomparable or have identical objective values. If $N \ge S + U$, then the following survival property holds.

  Assume that at some time~$t$ the combined parent and offspring population~$R_t$ contains a solution~$x$ that is contained in the first front~$F_1$ of the non-dominated sorting of~$R_t$. Then its objective value survives into the next generation, i.e., surely, $P_{t+1}$ contains an individual~$y$ such that $f(y)=f(x)$.
\end{lemma}

The results above show that the balanced \nsga does not have the efficiency problems of the classic \nsga for even numbers $m \ge 4$ of objectives. Since \citet{ZhengD24many} show that already the case $m=3$ is problematic for the classic \nsga, we show that the balanced \nsga optimizes \omm for three objectives also efficiently.
The first objective counts the number of zeros in the argument $x \in \{0, 1\}^n$; the second and third objectives count the numbers of ones in the first and second half of~$x$, resp.

\begin{theorem}
  \label{thm:runtimeBalancedNSGAThreeObjectiveOMM}
  Consider optimizing the $3$-objective \omm problem via the balanced \nsga with population size $N \ge (\frac n2+1)^2 + 4n+6$. Then after an expected number of at most $2en(\frac n2 + 1)^2$ iterations, the Pareto front is found.
\end{theorem}

\section{Runtime Analysis on Bi-Objective \omm}
\label{sec:balancedNSGAOneMinMax}

We bound the expected runtime of the balanced \nsga on \omm by $O(n + \frac{n^2\log{n}}{N})$ iterations, hence $O(Nn + n^2\log{n})$ function evaluations, when the population size at least four times the Pareto front size (\Cref{th:runtimeBalancedNSGAOneMinMax}). This bound is $O(n^2 \log n)$ function evaluations when $N = O(n \log n)$.

To put this result into perspective, we note that the classic \nsga, again for $N \ge 4(n+1)$, satisfies the guarantee of $O(n \log n)$ iterations, that is, $O(Nn\log n)$ function evaluations \cite{ZhengD23aij}. This bound is asymptotically tight for all $N \le n^{2-\eps}$, $\eps > 0$ any constant \cite{DoerrQ23LB}. Hence the classic \nsga obtains a $\Theta(n^2 \log n)$ runtime (function evaluations) only with the smallest admissible population size of $\Theta(n)$.
Recently, a bound of $O(Nn\log n)$ function evaluations was also proven for the SPEA2~\cite{RenBLQ24}.
For completeness, we note that the simplistic SEMO algorithm finds the full Pareto front of \omm in $O(n^2 \log n)$ iterations and function evaluations \cite{GielL10}. This result can easily be extended to the GSEMO algorithm. A matching lower bound of $\Omega(n^2 \log n)$ was shown for the SEMO in \cite{OsunaGNS20} and for the GSEMO in \cite{BossekS24}. An upper bound of $O(\mu n \log n)$ function evaluations was shown for the hypervolume-based $(\mu+1)$ SIBEA with $\mu \ge n+1$~\cite{NguyenSN15}.

As all individuals are Pareto-optimal for \omm, the runtime follows from how fast the algorithm spreads its population on the Pareto front.
We bound this time by considering the extremities of the currently covered Pareto front, i.e., the individuals with the largest number of~$1$s or of~$0$s.
Those are turned into individuals with one more~$1$ or~$0$, respectively, within about~$n$ iterations in expectation, requiring only a single bit flip.
The \emph{balance} property of the algorithm guarantees that the number of individuals at the extremities stays at about~$\frac{N}{n + 1}$ (\Cref{lem:numberOfElementsBoundOMMM}).
This number is quickly reached and then used to expand the Pareto front (\Cref{lem:discoverNewElement}).

We also prove a general result that bounds the expected time to cover certain parts of the Pareto front (\Cref{th:runtimeGeneralBalancedNSGAOneMinMax}).

We use Lemma~$1$ from \citet{ZhengLD22}, which also applies to the balanced \nsga, as it does not impose any restrictions on how to choose from individuals with the same CD.
The lemma states that a Pareto-optimal objective value in the population is never lost.

The following lemma is a direct application of \Cref{lem:numberOfElementsBound}.
It shows that the population maintains all individuals per objective value it found so far, up to a bound of $\lfloor \frac{N}{n+1} \rfloor - 4$.

\begin{lemma}
  \label{lem:numberOfElementsBoundOMMM}
  Consider the balanced \nsga with population size $N \ge 4(n + 1)$ on \omm.
  Then, for each objective value $(k, n - ~k)$ with $k \in [0..n]$, from the individuals $x$ in $R_t$ with $f(x) = (k, n - k)$, the population $P_{t + 1}$ contains at least $\min( \lfloor \frac{N}{n+1} \rfloor - 4, \card{\{x \in R_t \mid f(x) = (k, n - k)\}})$.
\end{lemma}

\Cref{lem:numberOfElementsBoundOMMM} shows that the population can maintain subpopulations of a size about~$\frac{N}{n + 1}$.
Once such a size is reached, there is a decent chance to extend the Pareto front.
The following lemma formalizes how quickly this happens.
\begin{lemma}
  \label{lem:discoverNewElement}
  Consider the balanced \nsga with population size $N \ge 4(n+1)$ on \omm. For $v \in [1..n]$ and $i \in \{1, 2\}$, let~$T^i_v$ denote the number of iterations needed, starting with an individual $x_0$ in the parent population with $f_i(x_0) = v$, to obtain an individual $x_f$ in the resulting population such that $f_i(x_f) = v - 1$. Then, $E[T^i_v] = O(\log{\lceil \frac{n}{v} \rceil} + \frac{n^2}{Nv} + 1)$.
\end{lemma}

We prove \Cref{lem:discoverNewElement} via the multiplicative up-drift theorem by \citet[Theorem~$3$]{DoerrK21algo}.
This theorem provides a bound on the expected number of steps for a random process to grow to a certain number if it increases in every single step by a multiple of its expected value.

Using \Cref{lem:discoverNewElement} lets us prove the following more general result of the bound for expanding the Pareto front.
It shows how quickly the population expands on a symmetric portion of the Pareto front, centered around $n/2$.
\begin{theorem}
  \label{th:runtimeGeneralBalancedNSGAOneMinMax}
  Consider the balanced \nsga with population size $N \ge 4(n+1)$ optimizing \omm. Let $\alpha \in  [0..\lfloor\frac{n}{2}\rfloor]$, and assume that there exists an $x \in P_0$ with $ \card{x}_1 \in [\alpha, n - \alpha]$. Then the expected number of iterations to cover $ \{ x \in \{0, 1\}^n \mid \card{x}_1 \in [\alpha, n - \alpha]\}$ is $O(n + \frac{n^2 \log{n}} {N})$.
\end{theorem}

For $\alpha = 0$, since the Pareto front covers the entire population, we get the following runtime bound for \omm.
\begin{corollary}
  \label{th:runtimeBalancedNSGAOneMinMax}
  The expected runtime of balanced \nsga with population size $N \ge 4\left(n+1\right)$ optimizing \omm is $O(n + \frac{n^2 \log{n}} {N})$ iterations, i.e., $O\left(nN + {n^2 \log{n}}\right)$ expected function evaluations.
\end{corollary}

\section{Runtime Analysis on Bi-Objective \ojzj}
\label{sec:balancedNSGAonOJZJ}

We analyze the runtime of the balanced \nsga on $\ojzj_k$ with $k = [2 .. \frac{n}{2}]$. We show that this time is $O(n + \frac{n^{k+1}}{N})$ iterations and thus $O(Nn +  n^{k+1} )$ function evaluations when the population size is at least four times the size of the Pareto front (\Cref{th:runtimeBalancedNSGAOJZJ}).
This guarantee is $O(n^{k + 1})$ function evaluations when $N = O(n^k)$, exhibiting a large parameter range with the asymptotically best runtime guarantee.

Previous results on the runtime of MOEAs on this benchmark include an $O(n^{k+1})$ iterations and function evaluations guarantee for the GSEMO and the result that the SEMO cannot optimize this benchmark~\cite{DoerrZ21aaai}. An upper bound of $O(n^k)$ iterations, hence $O(N n^k)$ function evaluations, was shown for the classic \nsga with population size at least four times the size of the Pareto front \cite{DoerrQ23tec}. The latter bound is tight apart from constant factors when $N = o(n^2 / k^2)$ \cite{DoerrQ23LB}.
The same bound $O(N n^{k + 1})$ was recently proven for the SPEA2~\cite{RenBLQ24}.
For the SMS-EMOA with population size $\mu \ge n - 2k + 3$, the bounds of $O(n^k)$ and $\Omega(n^k / \mu)$ iterations, hence $O(\mu n^k)$ and $\Omega(n^k)$ function evaluations, were shown by Bian et~al.~\shortcite{BianZLQ23}. The authors also show that a stochastic population update reduces the runtime by a factor of order $\min\{1, \mu / 2^{k/4}\}$.
\cite{RenQBQ24} show an expected runtime of $O(N^2 4^k + N n \log n)$ for a modified version of the \nsga that reorders individuals by maximizing the Hamming distance and uses crossover.

In our analysis of the balanced \nsga, we follow the approach of \citet{DoerrQ23tec}, who analyzed how the classic \nsga optimizes this benchmark.
This means we split the analysis into three stages.
The first stage bounds the time to find a solution on the inner Pareto front.
The second stage bounds the time to cover the inner Pareto front.
The third stage bounds the time to cover the outer Pareto front, which consists of only two objective values.
We show in the appendix that once the algorithm enters a later stage, it does not return to an earlier one. Thus, we bound the expected number of iterations by separately analyzing each stage.

For stage~$1$, we find that with very high probability at least one of the initial individuals is on the inner Pareto front.
This leads, in expectation, to a constant length of this stage.

\begin{lemma}
  \label{lem:stageOne}
  Regardless of the population size $N$ and of the initial population, for all $k \in [2..n/2]$, stage~$1$ needs an expected number of at most $\frac eN k^k + 1$ iterations.
\end{lemma}

For stage~$2$, as $\ojzj_k$ and \omm are similar, we apply \Cref{th:runtimeGeneralBalancedNSGAOneMinMax} with $\alpha = k$, resulting in the following lemma.

\begin{lemma}\label{lem:ojzjstage2}
  Using population size $N \ge 4(n - 2k + 3)$, stage~$2$ needs in expectation $O(n + \frac{n^2 \log n}{N})$ iterations.
\end{lemma}

For stage $3$, the arguments follow those for \Cref{lem:discoverNewElement}, but now with a  \emph{jump} of $k$ bits instead of $1$.
\begin{lemma}
  \label{lem:stageThree}
  Using $N \ge 4(n - 2k + 3)$, stage~$3$ needs in expectation $O(\log \frac{N}{n}) + 2 + \frac{8en^{k}(n+1)}{N}$ iterations.
\end{lemma}

Combining the results for all three stages, we obtain the following runtime guarantee.

\begin{theorem}
  \label{th:runtimeBalancedNSGAOJZJ}
  The expected runtime of the balanced \nsga with population size $N \ge 4(n - 2k + 3)$ on $\ojzj_k$ with $k \in [2 .. \frac{n}{2}]$ is $O(n + n^{k+1}/N)$ iterations, i.e., $O(Nn + n^{k+1})$ expected function evaluations.
\end{theorem}

\section{Runtime Analysis on Bi-Objective \lotz}
\label{sec:balancedNSGAonLeadingOnesTrailingZeros}

We bound the expected runtime of the balanced \nsga on \lotz by $O(\frac{n^3}{N} + n \log \frac{N}{n + 1})$ iterations, i.e., $O(n^3 + N n \log \frac{N}{n + 1})$ function evaluations, when the population size is at least four times the Pareto front size (\Cref{thm:runtimeBalancedNSGALOTZ}).
This bound is $O(n^3)$ function evaluations for $N = O(\frac{n^2}{\log n})$.

The following runtime results are known for \lotz. The SEMO takes $\Theta(n^3)$ iterations and function evaluations~\cite{LaumannsTZ04}.
For the GSEMO, \citet{Giel03} showed an upper bound of $O(n^3)$ iterations and evaluations, a matching lower bound was proven only for unrealistically small mutation rates \cite{DoerrKV13}.
An upper bound of $O(\mu n^2)$ iterations and function evaluations was shown for the $(\mu+1)$ SIBEA with population size $\mu \ge n+1$~\cite{BrockhoffFN08}. The classic \nsga with population size at least $4(n+1)$ solves the \lotz problem in $O(n^2)$ iterations and thus $O(N n^2)$ function evaluations.
The same bound was recently proven for the SPEA2~\cite{RenBLQ24}.
For the SMS-EMOA with population size $\mu \ge n+1$, a runtime guarantee of $O(\mu n^2)$ iterations and function evaluations was given by \citet{ZhengD24}.

Our analysis considers two phases.
The first phase bounds the time until the current population contains the all-$1$s bit string, which is Pareto-optimal.
The second phase bounds the remaining time until the entire Pareto front is covered.
Both phases take about the same time in expectation.

During the first phase, we consider individuals with an increasing prefix of~$1$s.
In the second phase, we consider individuals with an increasing suffix of~$0$s.
Each such improvement denotes a segment.
Each segment consists of the following two steps:
(1) We bound the time until the population contains about~$\frac{N}{n + 1}$ individuals that can easily be turned into improving offspring.
This step takes $O(\log \frac{N}{n + 1})$ in expectation (\Cref{lem:lotzQuicklyIncreasingPopulation}).
(2) We bound the time to create an improving offspring by $O(\frac{n^2}{N})$ in expectation.
\Cref{thm:runtimeBalancedNSGALOTZ} then follows since there are at most~$n$ segments per phase.

\begin{lemma}
  \label{lem:lotzQuicklyIncreasingPopulation}
  Consider the balanced \nsga with population size $N \geq 4(n + 1)$ optimizing the $\lotz \eqqcolon f$ function.
  Let~$t_0$ be any iteration.
  Furthermore, for all $t \in N$, let $v = \max_{y \in R_{t_0 + t}} f_1(y)$ and $Y_t = \{y \in R_{t_0 + t} \mid f_1(y) = v\}$.
  Last, let~$T$ denote the first iteration $t \in \N$ such that $|Y_{t}| \geq \max(1, \lfloor \frac{N}{n + 1} \rfloor - 4) \eqqcolon B$ or such that there is a $z \in R_{t + t_0}$ with $f_1(z) > v$.
  Then $E[T \mid t_0] = O(\log B)$.

  The same statement holds when exchanging all occurrences of~$f_1$ above by~$f_2$.
\end{lemma}

\Cref{lem:lotzQuicklyIncreasingPopulation} is sufficient to prove our main result.

\begin{theorem}
  \label{thm:runtimeBalancedNSGALOTZ}
  The expected runtime of the balanced NSGA-II with $N \geq 4(n + 1)$ on \lotz is $O(\frac{n^3}{N} + n \log \frac{N}{n + 1})$ iterations, i.e., $O(n^3 + N n \log \frac{N}{n + 1})$ function evaluations.
\end{theorem}

\section{Empirical Runtime Analysis}
\label{sec:experiments}

We complement our theoretical results with experiments, aiming to see how much the population size of the \nsga actually influences the number of function evaluations.
Our code is publicly available~\cite{DoerrIK24AaaiCodeRepo}.

We run the classic and the balanced \nsga on \omm, for problem sizes $n \in 10 \cdot [3 .. 12]$ and three population sizes~$N$.
For each combination of~$n$ and~$N$ per algorithm, we start~$50$ independent runs and log the number of function evaluations until the Pareto front is covered for the first time.
Let $M = n + 1$ denote the size of the Pareto front.
We consider the choices $N \in \{2M, 4M, 8M, 16M\}$, noting that our theoretical result holds for all $N \geq 4M$ (\Cref{th:runtimeBalancedNSGAOneMinMax}).

\begin{figure}
  \includegraphics[width = \columnwidth]{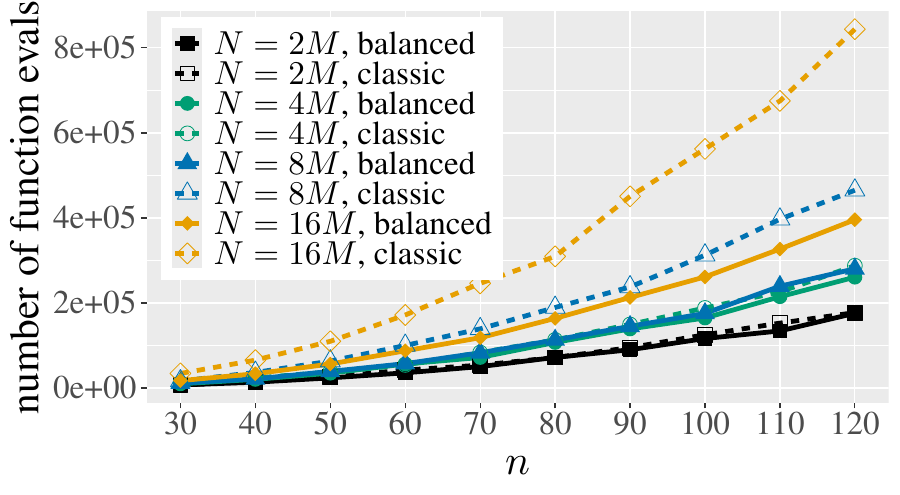}
  \caption{
    The average number of function evaluations of the classic (dashed lines) and the balanced (solid lines) \nsga optimizing \ommLong, for the shown population sizes~$N$ and problem sizes~$n$.
    The value~$M$ denotes the size of the Pareto front, i.e., $M = n + 1$.
    Each point is the average of~$50$ independent runs.
  }
  \label{fig:ommJointPlot}
\end{figure}

\Cref{fig:ommJointPlot} shows that the runtime increases for both algorithms for increasing~$N$.
This additional cost is larger for the classic \nsga than for the balanced one, statistically significantly for $N \in \{8M, 16M\}$.
Hence, not choosing the optimal population size for the classic \nsga is already for constant factors more penalized than for the balanced \nsga.
For $N \in \{2M, 4M\}$, the runtime of the classic and the balanced \nsga is roughly the same.
This shows that the classic \nsga still can perform very well for a careful choice of~$N$.
If the optimal choice for~$N$ is unknown, the balanced \nsga is a more robust choice.

We note that qualitatively very similar results hold for \ojzj and \lotz (as shown in the appendix).
Moreover, we show in the appendix that the balanced \nsga optimizes the $4$-objective \omm problem quickly.

\section{Conclusion}
\label{sec:conclusion}

We propose and analyze a simple tie-breaking modification to the classic \nsga, aiming to distribute individuals more evenly among different objective values with the same non-dominated rank and crowding distance.
Our theoretical results prove two major advantages of this modification:
(1) It is capable of efficiently optimizing multi-objective functions with at least three objectives for which the classic \nsga has at least an exponential expected runtime.
(2) Not choosing an optimal population size is far less important and leads for certain ranges to no asymptotic runtime loss.
Our experiments show that this effect can be already very clear when choosing slightly sub-optimal population sizes.

For future work, it would be interesting to prove lower bounds for the bi-objective scenarios and to improve the bounds in the many-objective setting.
Furthermore, we are optimistic that similar results could be obtained for other MOEAs that resort to random tie-breaking at some stage, for example, the \nsgathree \cite{DebJ14} and the SMS-EMOA \cite{BeumeNE07}.

}%sloppy, please do not remove

\cleardoublepage
\subsection*{Acknowledgments}
This research benefited from the support of the FMJH Program Gaspard Monge for optimization and operations research and their interactions with data science.

\bibliography{ich_master.bib, alles_ea_master.bib, rest.bib}

\begin{thebibliography}{53}
\providecommand{\natexlab}[1]{#1}

\bibitem[{Auger and Doerr(2011)}]{AugerD11}
Auger, A.; and Doerr, B., eds. 2011.
\newblock \emph{Theory of Randomized Search Heuristics}.
\newblock World Scientific Publishing.

\bibitem[{Beume, Naujoks, and Emmerich(2007)}]{BeumeNE07}
Beume, N.; Naujoks, B.; and Emmerich, M. 2007.
\newblock {SMS-EMOA}: Multiobjective selection based on dominated hypervolume.
\newblock \emph{European Journal of Operational Research}, 181: 1653--1669.

\bibitem[{Bian and Qian(2022)}]{BianQ22}
Bian, C.; and Qian, C. 2022.
\newblock Better running time of the non-dominated sorting genetic
  algorithm~{II} {(NSGA-II)} by using stochastic tournament selection.
\newblock In \emph{Parallel Problem Solving From Nature, PPSN 2022}, 428--441.
  Springer.

\bibitem[{Bian, Qian, and Tang(2018)}]{BianQT18ijcaigeneral}
Bian, C.; Qian, C.; and Tang, K. 2018.
\newblock A general approach to running time analysis of multi-objective
  evolutionary algorithms.
\newblock In \emph{International Joint Conference on Artificial Intelligence,
  {IJCAI} 2018}, 1405--1411. {ijcai.org}.

\bibitem[{Bian et~al.(2024)Bian, Ren, Li, and Qian}]{BianRLQ24}
Bian, C.; Ren, S.; Li, M.; and Qian, C. 2024.
\newblock An archive can bring provable speed-ups in multi-objective
  evolutionary algorithms.
\newblock In \emph{International Joint Conference on Artificial Intelligence,
  {IJCAI} 2024}, 6905--6913. ijcai.org.

\bibitem[{Bian et~al.(2023)Bian, Zhou, Li, and Qian}]{BianZLQ23}
Bian, C.; Zhou, Y.; Li, M.; and Qian, C. 2023.
\newblock Stochastic population update can provably be helpful in
  multi-objective evolutionary algorithms.
\newblock In \emph{International Joint Conference on Artificial Intelligence,
  IJCAI 2023}, 5513--5521. ijcai.org.

\bibitem[{Bossek and Sudholt(2024)}]{BossekS24}
Bossek, J.; and Sudholt, D. 2024.
\newblock Runtime analysis of quality diversity algorithms.
\newblock \emph{Algorithmica}, 86: 3252--3283.

\bibitem[{Brockhoff, Friedrich, and Neumann(2008)}]{BrockhoffFN08}
Brockhoff, D.; Friedrich, T.; and Neumann, F. 2008.
\newblock Analyzing hypervolume indicator based algorithms.
\newblock In \emph{Parallel Problem Solving from Nature, {PPSN} 2008},
  651--660. Springer.

\bibitem[{Cerf et~al.(2023)Cerf, Doerr, Hebras, Kahane, and
  Wietheger}]{CerfDHKW23}
Cerf, S.; Doerr, B.; Hebras, B.; Kahane, J.; and Wietheger, S. 2023.
\newblock The first proven performance guarantees for the {N}on-{D}ominated
  {S}orting {G}enetic {A}lgorithm {II} ({NSGA-II}) on a combinatorial
  optimization problem.
\newblock In \emph{International Joint Conference on Artificial Intelligence,
  {IJCAI} 2023}, 5522--5530. ijcai.org.

\bibitem[{Coello, Lamont, and van Veldhuizen(2007)}]{CoelloLV07}
Coello, C. A.~C.; Lamont, G.~B.; and van Veldhuizen, D.~A. 2007.
\newblock \emph{Evolutionary Algorithms for Solving Multi-Objective Problems}.
\newblock Springer, 2nd edition.

\bibitem[{Covantes~Osuna et~al.(2020)Covantes~Osuna, Gao, Neumann, and
  Sudholt}]{OsunaGNS20}
Covantes~Osuna, E.; Gao, W.; Neumann, F.; and Sudholt, D. 2020.
\newblock Design and analysis of diversity-based parent selection schemes for
  speeding up evolutionary multi-objective optimisation.
\newblock \emph{Theoretical Computer Science}, 832: 123--142.

\bibitem[{Dang et~al.(2023{\natexlab{a}})Dang, Opris, Salehi, and
  Sudholt}]{DangOSS23gecco}
Dang, D.-C.; Opris, A.; Salehi, B.; and Sudholt, D. 2023{\natexlab{a}}.
\newblock Analysing the robustness of {NSGA-II} under noise.
\newblock In \emph{Genetic and Evolutionary Computation Conference, GECCO
  2023}, 642--651. {ACM}.

\bibitem[{Dang et~al.(2023{\natexlab{b}})Dang, Opris, Salehi, and
  Sudholt}]{DangOSS23aaai}
Dang, D.-C.; Opris, A.; Salehi, B.; and Sudholt, D. 2023{\natexlab{b}}.
\newblock A proof that using crossover can guarantee exponential speed-ups in
  evolutionary multi-objective optimisation.
\newblock In \emph{Conference on Artificial Intelligence, {AAAI} 2023},
  12390--12398. {AAAI} Press.

\bibitem[{Deb and Jain(2014)}]{DebJ14}
Deb, K.; and Jain, H. 2014.
\newblock An evolutionary many-objective optimization algorithm using
  reference-point-based nondominated sorting approach, part~{I:} solving
  problems with box constraints.
\newblock \emph{IEEE Transactions on Evolutionary Computation}, 18: 577--601.

\bibitem[{Deb et~al.(2002)Deb, Pratap, Agarwal, and Meyarivan}]{DebPAM02}
Deb, K.; Pratap, A.; Agarwal, S.; and Meyarivan, T. 2002.
\newblock A fast and elitist multiobjective genetic algorithm: {NSGA-II}.
\newblock \emph{IEEE Transactions on Evolutionary Computation}, 6: 182--197.

\bibitem[{Deng et~al.(2024)Deng, Zheng, Li, Liu, and Doerr}]{DengZLLD24}
Deng, R.; Zheng, W.; Li, M.; Liu, J.; and Doerr, B. 2024.
\newblock Runtime analysis for state-of-the-art multi-objective evolutionary
  algorithms on the subset selection problem.
\newblock In Affenzeller, M.; Winkler, S.~M.; Kononova, A.~V.; Trautmann, H.;
  Tusar, T.; Machado, P.; and B{\"{a}}ck, T., eds., \emph{Parallel Problem
  Solving from Nature, {PPSN} 2024, Part~{III}}, 264--279. Springer.

\bibitem[{Doerr, Ivan, and Krejca(2024)}]{DoerrIK24AaaiCodeRepo}
Doerr, B.; Ivan, T.; and Krejca, M.~S. 2024.
\newblock Speeding Up the NSGA-II With a Simple Tie-Breaking Rule (Code).
\newblock Zenodo. \url{https://doi.org/10.5281/zenodo.14501034}.

\bibitem[{Doerr, Kodric, and Voigt(2013)}]{DoerrKV13}
Doerr, B.; Kodric, B.; and Voigt, M. 2013.
\newblock Lower bounds for the runtime of a global multi-objective evolutionary
  algorithm.
\newblock In \emph{Congress on Evolutionary Computation, CEC 2013}, 432--439.
  IEEE.

\bibitem[{Doerr, Korkotashvili, and Krejca(2024)}]{DoerrKK24arxiv}
Doerr, B.; Korkotashvili, D.; and Krejca, M.~S. 2024.
\newblock Difficulties of the NSGA-II with the Many-Objective LeadingOnes
  Problem.
\newblock \emph{CoRR}, abs/2411.10017.

\bibitem[{Doerr and K{\"{o}}tzing(2021)}]{DoerrK21algo}
Doerr, B.; and K{\"{o}}tzing, T. 2021.
\newblock Multiplicative up-drift.
\newblock \emph{Algorithmica}, 83: 3017--3058.

\bibitem[{Doerr and K{\"{u}}nnemann(2015)}]{DoerrK15}
Doerr, B.; and K{\"{u}}nnemann, M. 2015.
\newblock Optimizing linear functions with the $(1+\lambda)$ evolutionary
  algorithm---different asymptotic runtimes for different instances.
\newblock \emph{Theoretical Computer Science}, 561: 3--23.

\bibitem[{Doerr and Neumann(2020)}]{DoerrN20}
Doerr, B.; and Neumann, F., eds. 2020.
\newblock \emph{Theory of Evolutionary Computation---Recent Developments in
  Discrete Optimization}.
\newblock Springer.
\newblock Also available at
  \url{http://www.lix.polytechnique.fr/Labo/Benjamin.Doerr/doerr_neumann_book.html}.

\bibitem[{Doerr and Qu(2023{\natexlab{a}})}]{DoerrQ23tec}
Doerr, B.; and Qu, Z. 2023{\natexlab{a}}.
\newblock A first runtime analysis of the {NSGA-II} on a multimodal problem.
\newblock \emph{IEEE Transactions on Evolutionary Computation}, 27: 1288--1297.

\bibitem[{Doerr and Qu(2023{\natexlab{b}})}]{DoerrQ23LB}
Doerr, B.; and Qu, Z. 2023{\natexlab{b}}.
\newblock From understanding the population dynamics of the {NSGA-II} to the
  first proven lower bounds.
\newblock In \emph{Conference on Artificial Intelligence, {AAAI} 2023},
  12408--12416. {AAAI} Press.

\bibitem[{Doerr and Qu(2023{\natexlab{c}})}]{DoerrQ23crossover}
Doerr, B.; and Qu, Z. 2023{\natexlab{c}}.
\newblock Runtime analysis for the {NSGA-II:} Provable speed-ups from
  crossover.
\newblock In \emph{Conference on Artificial Intelligence, {AAAI} 2023},
  12399--12407. {AAAI} Press.

\bibitem[{Doerr and Zheng(2021)}]{DoerrZ21aaai}
Doerr, B.; and Zheng, W. 2021.
\newblock Theoretical analyses of multi-objective evolutionary algorithms on
  multi-modal objectives.
\newblock In \emph{Conference on Artificial Intelligence, {AAAI} 2021},
  12293--12301. {AAAI} Press.

\bibitem[{Droste, Jansen, and Wegener(2002)}]{DrosteJW02}
Droste, S.; Jansen, T.; and Wegener, I. 2002.
\newblock On the analysis of the (1+1) evolutionary algorithm.
\newblock \emph{Theoretical Computer Science}, 276: 51--81.

\bibitem[{Fortin and Parizeau(2013)}]{FortinP13}
Fortin, F.; and Parizeau, M. 2013.
\newblock Revisiting the {NSGA-II} crowding-distance computation.
\newblock In \emph{Genetic and Evolutionary Computation Conference, {GECCO}
  2013}, 623--630. {ACM}.

\bibitem[{Giel(2003)}]{Giel03}
Giel, O. 2003.
\newblock Expected runtimes of a simple multi-objective evolutionary algorithm.
\newblock In \emph{Congress on Evolutionary Computation, {CEC} 2003},
  1918--1925. {IEEE}.

\bibitem[{Giel and Lehre(2010)}]{GielL10}
Giel, O.; and Lehre, P.~K. 2010.
\newblock On the effect of populations in evolutionary multi-objective
  optimisation.
\newblock \emph{Evolutionary Computation}, 18: 335--356.

\bibitem[{Huang et~al.(2021)Huang, Zhou, Luo, and Lin}]{HuangZLL21}
Huang, Z.; Zhou, Y.; Luo, C.; and Lin, Q. 2021.
\newblock A runtime analysis of typical decomposition approaches in {MOEA/D}
  framework for many-objective optimization problems.
\newblock In \emph{International Joint Conference on Artificial Intelligence,
  {IJCAI} 2021}, 1682--1688.

\bibitem[{Jansen(2013)}]{Jansen13}
Jansen, T. 2013.
\newblock \emph{Analyzing Evolutionary Algorithms -- The Computer Science
  Perspective}.
\newblock Springer.

\bibitem[{Jansen, Jong, and Wegener(2005)}]{JansenJW05}
Jansen, T.; Jong, K. A.~D.; and Wegener, I. 2005.
\newblock On the choice of the offspring population size in evolutionary
  algorithms.
\newblock \emph{Evolutionary Computation}, 13: 413--440.

\bibitem[{Laumanns, Thiele, and Zitzler(2004)}]{LaumannsTZ04}
Laumanns, M.; Thiele, L.; and Zitzler, E. 2004.
\newblock Running time analysis of multiobjective evolutionary algorithms on
  pseudo-{B}oolean functions.
\newblock \emph{IEEE Transactions on Evolutionary Computation}, 8: 170--182.

\bibitem[{Laumanns et~al.(2002)Laumanns, Thiele, Zitzler, Welzl, and
  Deb}]{LaumannsTZWD02}
Laumanns, M.; Thiele, L.; Zitzler, E.; Welzl, E.; and Deb, K. 2002.
\newblock Running time analysis of multi-objective evolutionary algorithms on a
  simple discrete optimization problem.
\newblock In \emph{Parallel Problem Solving from Nature, {PPSN} 2002}, 44--53.
  Springer.

\bibitem[{Li et~al.(2016)Li, Zhou, Zhan, and Zhang}]{LiZZZ16}
Li, Y.-L.; Zhou, Y.-R.; Zhan, Z.-H.; and Zhang, J. 2016.
\newblock A primary theoretical study on decomposition-based multiobjective
  evolutionary algorithms.
\newblock \emph{IEEE Transactions on Evolutionary Computation}, 20: 563--576.

\bibitem[{Neumann and Witt(2010)}]{NeumannW10}
Neumann, F.; and Witt, C. 2010.
\newblock \emph{Bioinspired Computation in Combinatorial Optimization --
  Algorithms and Their Computational Complexity}.
\newblock Springer.

\bibitem[{Nguyen, Sutton, and Neumann(2015)}]{NguyenSN15}
Nguyen, A.~Q.; Sutton, A.~M.; and Neumann, F. 2015.
\newblock Population size matters: rigorous runtime results for maximizing the
  hypervolume indicator.
\newblock \emph{Theoretical Computer Science}, 561: 24--36.

\bibitem[{Opris et~al.(2024)Opris, Dang, Neumann, and Sudholt}]{OprisDNS24}
Opris, A.; Dang, D.~C.; Neumann, F.; and Sudholt, D. 2024.
\newblock Runtime analyses of {NSGA-III} on many-objective problems.
\newblock In \emph{Genetic and Evolutionary Computation Conference, GECCO
  2024}, 1596--1604. {ACM}.

\bibitem[{Ren et~al.(2024{\natexlab{a}})Ren, Bian, Li, and Qian}]{RenBLQ24}
Ren, S.; Bian, C.; Li, M.; and Qian, C. 2024{\natexlab{a}}.
\newblock A first running time analysis of the {S}trength {P}areto
  {E}volutionary {A}lgorithm~2 {(SPEA2)}.
\newblock In \emph{Parallel Problem Solving from Nature, {PPSN} 2024, Part
  {III}}, 295--312. Springer.

\bibitem[{Ren et~al.(2024{\natexlab{b}})Ren, Qiu, Bian, Li, and
  Qian}]{RenQBQ24}
Ren, S.; Qiu, Z.; Bian, C.; Li, M.; and Qian, C. 2024{\natexlab{b}}.
\newblock Maintaining diversity provably helps in evolutionary multimodal
  optimization.
\newblock In \emph{International Joint Conference on Artificial Intelligence,
  {IJCAI} 2024}, 7012--7020. ijcai.org.

\bibitem[{Rudolph(1997)}]{Rudolph97}
Rudolph, G. 1997.
\newblock \emph{Convergence Properties of Evolutionary Algorithms}.
\newblock Verlag Dr.~Kov{\v a}c.

\bibitem[{Thierens(2003)}]{Thierens03}
Thierens, D. 2003.
\newblock Convergence time analysis for the multi-objective counting ones
  problem.
\newblock In \emph{Evolutionary Multi-Criterion Optimization, {EMO} 2003},
  355--364. Springer.

\bibitem[{Wietheger and Doerr(2023)}]{WiethegerD23}
Wietheger, S.; and Doerr, B. 2023.
\newblock A mathematical runtime analysis of the {N}on-dominated {S}orting
  {G}enetic {A}lgorithm {III} ({NSGA-III}).
\newblock In \emph{International Joint Conference on Artificial Intelligence,
  {IJCAI} 2023}, 5657--5665. ijcai.org.

\bibitem[{Wietheger and Doerr(2024)}]{WiethegerD24}
Wietheger, S.; and Doerr, B. 2024.
\newblock Near-tight runtime guarantees for many-objective evolutionary
  algorithms.
\newblock In \emph{Parallel Problem Solving from Nature, {PPSN} 2024,
  Part~{IV}}, 153--168. Springer.

\bibitem[{Witt(2006)}]{Witt06}
Witt, C. 2006.
\newblock Runtime analysis of the ($\mu$ + 1) {EA} on simple pseudo-{B}oolean
  functions.
\newblock \emph{Evolutionary Computation}, 14: 65--86.

\bibitem[{Zheng and Doerr(2023)}]{ZhengD23aij}
Zheng, W.; and Doerr, B. 2023.
\newblock Mathematical runtime analysis for the non-dominated sorting genetic
  algorithm {II} ({NSGA-II}).
\newblock \emph{Artificial Intelligence}, 325: 104016.

\bibitem[{Zheng and Doerr(2024{\natexlab{a}})}]{ZhengD24approx}
Zheng, W.; and Doerr, B. 2024{\natexlab{a}}.
\newblock Approximation guarantees for the {N}on-{D}ominated {S}orting
  {G}enetic {A}lgorithm {II} ({NSGA-II}).
\newblock \emph{IEEE Transactions on Evolutionary Computation}.
\newblock In press, \url{https://doi.org/10.1109/TEVC.2024.3402996}.

\bibitem[{Zheng and Doerr(2024{\natexlab{b}})}]{ZhengD24many}
Zheng, W.; and Doerr, B. 2024{\natexlab{b}}.
\newblock Runtime analysis for the {NSGA-II}: proving, quantifying, and
  explaining the inefficiency for many objectives.
\newblock \emph{IEEE Transactions on Evolutionary Computation}, 28: 1442--1454.

\bibitem[{Zheng and Doerr(2024{\natexlab{c}})}]{ZhengD24}
Zheng, W.; and Doerr, B. 2024{\natexlab{c}}.
\newblock Runtime analysis of the {SMS-EMOA} for many-objective optimization.
\newblock In \emph{Conference on Artificial Intelligence, {AAAI} 2024},
  20874--20882. {AAAI} Press.

\bibitem[{Zheng, Liu, and Doerr(2022)}]{ZhengLD22}
Zheng, W.; Liu, Y.; and Doerr, B. 2022.
\newblock A first mathematical runtime analysis of the {N}on-{D}ominated
  {S}orting {G}enetic {A}lgorithm~{II} ({NSGA-II}).
\newblock In \emph{Conference on Artificial Intelligence, {AAAI} 2022},
  10408--10416. {AAAI} Press.

\bibitem[{Zhou et~al.(2011)Zhou, Qu, Li, Zhao, Suganthan, and
  Zhang}]{ZhouQLZSZ11}
Zhou, A.; Qu, B.-Y.; Li, H.; Zhao, S.-Z.; Suganthan, P.~N.; and Zhang, Q. 2011.
\newblock Multiobjective evolutionary algorithms: A survey of the state of the
  art.
\newblock \emph{Swarm and Evolutionary Computation}, 1: 32--49.

\bibitem[{Zhou, Yu, and Qian(2019)}]{ZhouYQ19}
Zhou, Z.-H.; Yu, Y.; and Qian, C. 2019.
\newblock \emph{Evolutionary Learning: Advances in Theories and Algorithms}.
\newblock Springer.

\end{thebibliography}

\cleardoublepage
\section*{Appendix}

The section titles in this document share the names with the respective sections in the submission.

\subsection*{Benchmarks}

For \ojzj, we define the following partitions of its Pareto front and set, following the notation of \citet{DoerrZ21aaai}.
Let the inner part of the Pareto set be denoted by $S^*_{\mathrm{In}} = \{ x \mid \card{x}_1 \in [k..n-k]\}$, the outer part by $S^*_{\mathrm{Out}} = \{ x \mid \card{x}_1 \in \{0, n\}\}$, the inner part of the Pareto front by
$F^*_{\mathrm{In}} = \{ (a, 2k + n - a) \mid a \in [2k..n] \}$ and the outer part by $F^*_{\mathrm{Out}} = \{ (a, 2k + n - a) \mid a \in \{k, n+k \} \}$.

\subsection*{Properties of the Balanced \nsga}

\begin{proof}[Proof of \Cref{lem:max2mpositives}]
  We calculate the crowding distance for each individual in $F_{i^*}$.
  Let $A = (a_1, a_2, ..., a_m) \in f(F_{i^*})$.
  Assume there exists at least an individual $x \in R_t$ such that $f(x) = A$ (if not, the lemma holds trivially).
  For each $i \in [m]$ consider the sorted solution lists with respect to $f_i$ denoted as $S_{i.1}, S_{i.2}, ... S_{i.2N}$.
  Then there exist $l_i, r_i \in [2N]$ with $ l_i \le r_i$, and $r_i - l_i$ maximal, such that $\left[l_i..r_i\right] = \{j \mid f_i(S_{i.j}) = a_i\}$.
  If $l_i = 0$ or $l_i = 2N$, by definition of crowding distance we have $\cdis(S_{i.l_i}) = \infty > 0$.
  If $l_i \in [1..2N - 1]$, we have $\cdis(S_{i.l_i}) = \frac{f\left(S_{i.l_i+1}\right) - f\left(S_{i.l_i-1}\right)}{f\left(S_{i.2N}\right) - f\left(S_{i.1}\right)} \ge \frac{f\left(S_{i.l_i}\right) - f\left(S_{i.l_i-1}\right)}{f\left(S_{i.2N}\right) - f\left(S_{1.1}\right)} > 0$, from the definition of $l_i$.
  Similarly, we get that $\cdis(S_{i.r_i}) > 0$.
  We now have that $\cdis(S_{i.l_i}) > 0,\cdis(S_{i.r_i}) > 0 $, for each $i \in [m]$.
  For all $x \in  \left( \bigcap_{i \in [m]} \{S_{i.j} \mid j \in [l_i + 1..r_i - 1]\} \right)$  denote by $p_i$ its position in $S_i$.
  Hence, $p_i \in [l_i + 1..r_i - 1]$ for each $i \in [1, m]$.
  Thus, we know $f_i(S_{i.p_i-1}) = f_i(S_{i.p_i+1})$ for each $i \in [m]$ by the definition of $l_i$ and $r_i$.
  Therefore we get $\cdis(x) = 0$.

  Then the individuals in $R_t$ with objective value $A = (a_1, a_2, ..., a_m)$ and positive crowding distance are exactly $\bigcup _{i \in [m]}\{S_{i.l_i}, S_{i.r_i}\}$, which has at most $2m$ distinct elements.
\end{proof}

\begin{proof}[Proof of \Cref{lem:numberOfElementsBound}]
  If $\lfloor \frac{C}{S} \rfloor < 2m$, then the lemma says that we select at least~$0$ individuals, which is trivially true.
  Hence, we assume in the following that $\lfloor \frac{C}{S} \rfloor \geq 2m$.

  For any $A \in f(F_{i^*})$, denote by $\mathrm{CD}^*_{A}$ the number of individuals $y \in F_{i^*}$ such that $f(y) = A$ and $\cdis(y) \neq 0$, and let $\mathrm{CD^*} = \sum_{A \in f(F_{i^*})} \mathrm{CD}^*_A$.

  From the definition of crowding distance and \Cref{lem:max2mpositives} we know that $\mathrm{CD}^*_A \in [0..2m]$ for each $A \in F_{i^*}$.
  Therefore, $\mathrm{CD^*} = ~\sum_{A \in f(F_{i^*})} \mathrm{CD}^*_A \in~ \left[{0..2mS}\right]$.
  Hence, the number of remaining individuals with crowding distance $0$ to be selected is at least $C - 2mS$.
  By the property of the balanced \nsga, this means that for each objective value $A \in f(F_{i^*})$ there remain at least $\lfloor \frac{C}{S} \rfloor - 2m$ individuals with crowding distance $0$.
  However, if we do not have enough individuals $x$ in $F_{i^*}$ with $f(x) = A$  to \emph{fill in} all these spots, we keep all that we have, which is why we need a minimum between the two bounds.
\end{proof}

\subsection*{The Balanced \nsga Is Efficient For Three or More Objectives}

We state the formal definitions of the many-objective versions of the benchmarks we consider more precisely.
Let $n,m \in \N$ be such that $m$ is even and $m'=m/2$ divides $n$, that is, $n'=n/m'$ is integral. For a bit string $x \in \{0,1\}^n$ and $i \in [1..m']$ let $x_{B_i} = (x_{(i-1)n'+1}, \dots, x_{in'})$. Let $f$ be one of the \omm, \lotz, or $\ojzj_k$ problems defined on bit strings of length $n'$, where $k \in [2..n'/2]$ for $\ojzj_k$. Then the $m$-objective version $g$ of the bi-objective problem $f$ is defined by $g(x) = (f_1(x_{B_i}), f_2(x_{B_i}))_{i \in [m']}$ for all $x \in \{0,1\}^n$. It is easy to see that $x$ is Pareto-optimal for $g$ if and only if $x_{B_i}$ is Pareto-optimal for $f$ for all $i \in [1..m']$. Consequently, the Pareto front of $g$ is the $m'$-fold direct product of the Pareto front of $f$, and the cardinality of the former is the $m'$-th power of the cardinality of the latter, namely $M := (n'+1)^{m'}$ for \omm and \lotz and $M := (n'-2k+3)^{m'}$ for $\ojzj_k$.

\begin{proof}[Proof of \Cref{lem:survival}]
  If $|F_1| \le N$, then all of $F_1$ survives into $P_{t+1}$ and there is nothing to show. Hence assume that $|F_1| > N$. Then the critical rank is one and the next population is selected from $F_1$ according to crowding distance and, as tie breaker, multiplicity of objective values.

  We regard first the individuals with positive crowding distance. By assumption, their number $U_0$ is at most $U$. Since $U < N$, all these individuals are selected into $P_{t+1}$.

  In the second selection step, $N - U_0 \ge N - U \ge S$ individuals are selected in a balanced manner from the individuals with crowding distance zero. Since there are at most $S$ different objective values in $f(F_1)$, there are at most $S$ objective values covered by individuals with crowding distance zero. Since $N - U_0 \ge S$, for each of them at least one individual is selected in this second selection stage.

  In summary, each objective value survives into $P_{t+1}$.
\end{proof}

\begin{proof}[Proof of \Cref{thm:manyo}]
  Let $m' = m/2$ be the number of blocks in the definition of the $m$-objective versions of the problems regarded and $n' = n/m'$ be the length of each block.

  We start with the case of \omm. The size of the Pareto front is easily seen to be $M \coloneqq (n' + 1)^{m'}$. We further note that any solution is Pareto optimal, hence the non-dominated sorting ends with a single front $F_1$, whose rank is critical. Also, $S := M$ is also an upper bound for any set of pair-wise incomparable solutions. From~\cite[(1)]{ZhengD24many}, noting that each of the $m$ objectives of \omm takes exactly $n'+1$ different values, we know that the maximum number of individuals with positive crowding distance in a set of individuals such that any two are incomparable or have identical objective value, is at most $U\coloneqq 2m(n'+1)$.

  From Lemma~\ref{lem:survival}, using the values of $S$ and $U$ just defined, and our assumption on the size $N$ of the population, we see that solution values are never lost, that is, once the population contains an individual with a particular objective value $v \in [0..n']^{m'}$, it does so forever.

  With this information, we estimate the runtime by adding up the waiting times for generating a solution with an objective value not yet in the population. For this, we note that as long as $f(P_t)$ is not yet the full Pareto front, there exist $x \in P_t$ and $y \in \{0,1\}^n$ such that $\|x-y\|_1=1$ and $f(y) \notin f(P_t)$. The probability that a single application of the mutation operator to $x$ generates $y$ is $\frac 1n (1-\frac 1n)^{n-1} \ge \frac{1}{en}$. Hence the probability that none of the $N$ offspring generations in this iteration does not select $x$ as parent and then create $y$ from it, is at most $(1 - \frac 1N + \frac 1N (1-\frac{1}{en}))^N = (1 - \frac{1}{enN})^N \le \exp(-\frac{1}{en}) \le 1 - \frac 1 {2en}$. Consequently, with probability at least $\frac 1 {2en}$, the individual $y$ is generated and thus the corresponding objective value is covered by the population. Hence it takes an expected number of $2en$ iterations to cover an additional objective value, showing that the full expected runtime is $2enM$ iterations.

  The proof for our \lotz result reuses many arguments from the analysis of \omm. We take $S \le (n'+1)^{m-1}$ now. Since the range of the objective values is $[0..n']$ again, we can use Lemma~\ref{lem:survival} with the same values of $S$ and $U$ as before.
  Also, the expected waiting time for generating a particular neighbor of an existing solution is again $2en$ iterations. We also note that the maximum LeadingOnes value in the population is never decreasing -- a solution with this value maximal is only strictly dominated (and thus pushed onto a higher front) by a solution with larger LeadingOnes value. Hence the largest LeadingOnes value in the combined parent and offspring population always lies in the first front, and hence survives by Lemma~\ref{lem:survival}. Now with the previous argument, we see that it takes an expected number of at most $2en$ iterations to increase this maximum LeadingOnes value by one. Hence after at most $2en^2$ iterations, we have the all-ones string in the population.

  We note that this solution lies on the Pareto front. From it, by repeatedly generating suitable neighbors, we can compute the whole Pareto front. We do not lose solution values by Lemma~\ref{lem:survival}. Hence after another $2enM$ iterations, the full Pareto front is computed. This completes the proof for \lotz.

  For the $\ojzj_k$ problem, we first consider the maximum sum of objective values in the population, that is, $m_t = \max\{s(x) \mid x \in P_t\}$ with $s(x)=\sum_{i=1}^m f_i(x)$. The number $s(x)$ attains the maximum value $m'(n'+1)$ if and only if $x$ is a Pareto optimum, that is, $\|x_{B_i}\|_1 \in \{0\} \cup [k..n'-k] \cup \{n'\}$ for all $i \in [1..m']$.

  We first note that any individual $x$ with maximum $s$-value in $R_t$ is not strictly dominated by any other individual, hence $x \in F_1$ in the selection phase and thus its objective value survives by Lemma~\ref{lem:survival} (note that we can use the same $U$ as before since again each objective takes $n'$ different values). Consequently, $m_t$ is non-decreasing over time. Consider some iteration with $m_t < m'(n'+1)$. Then there are $x \in P_t$ with $s(x) = m_t$ and $y \in \{0,1\}^n$ with $\|x-y\|_1 = 1$ and $s(y) = s(x)+2$. As before, the expected time to generate (and thus let enter in the population) such a $y$ is at most $2en$ iterations. Hence after an expected time of at most $m'(k-1)2en$ iterations,  $m_t$ has reached the maximum possible value, that is, the population contains an individual on the Pareto front.

  From this point on, we again note that Pareto-optimal solution values remain in the population (Lemma~\ref{lem:survival}). While the Pareto front has not yet been fully computed, we always have some $x \in P$ and some $y \in \{0,1\}^n$ such that $\|x-y\|_1 \le k$ and $f(y) \notin f(P_t)$. The probability to generate $y$ in one iteration, analogous to the computation further above, is at least $1 / (2en^k)$. Hence an expected number of at most $2en^k M$ iterations suffice to compute the remainder of the Pareto front.
\end{proof}

\paragraph{Discussion of \Cref{thm:runtimeBalancedNSGAThreeObjectiveOMM}.}
The $3$-objective \omm benchmark was defined by \citet{ZhengD24many} as follows.
Let $n\in \N$ be even. Then the $3$-objective \omm function is $(f_1,f_2,f_3)\colon\{0,1\}^n\rightarrow \R^3$ with
\begin{align*}
  f_1(x) & = \sum\nolimits_{i \in [n]}(1-x_i), \textrm{ as well as} \\
  f_2(x) & = \sum\nolimits_{i \in [n/2]} x_i, \textrm{ and }
  f_3(x) = \sum\nolimits_{i \in [n/2+1 .. n]} x_i,
\end{align*}
for all $x=(x_1,\dots,x_n)\in\{0,1\}^n$.
Hence the first objective counts the number of zeros in the argument. The second and third objectives count the numbers of ones in the first and second half of it, respectively. Any $x\in\{0,1\}^n$ is a Pareto optimum of this problem. The Pareto front thus is $\{(n-v_2-v_3,v_2,v_3)\mid v_2,v_3\in [0..n']\}$, where we write $n'\coloneqq \frac n2$. The size of the Pareto front is $M=(\frac n2 +1)^2$ and this is also the size $S$ of a largest set of pair-wise incomparable solutions. Noting that the first objective takes $n+1$ values, whereas the second and third take $n'+1$ values, we see that we can take $U = 4n+6$ in Lemma~\ref{lem:survival}. With these preparations, a proof completely analogous to the one for even numbers of objectives (and thus omitted) yields the following result.

\subsection*{Runtime Analysis on Bi-Objective \omm}

\begin{lemma} [Lemma~$1$ from \cite{ZhengLD22}]
  \label{lem:keepParetoFrontOMMM}
  Consider one iteration of the balanced \nsga with population size $N \ge 4(n + 1)$, optimizing the \omm function. Assume that in some iteration $t \in \N$ the combined parent and offspring population $R_{t}$ contains a solution $x$ with objective value $(k, n - k)$ for some $k \in \left[0..n \right]$. Then also the next parent population $P_{t+1}$ contains an individual~$y$ with $f(y) = (k, n - k)$.
\end{lemma}

\begin{theorem}[{multiplicative up-drift~\cite[Theorem~$3$]{DoerrK21algo}}]
  \label{th:multiplicativeUpDriftTheorem}
  Let $(X_t)_{t\in\N}$ be a stochastic process over the positive integers. Assume that there are $B, k \in \Z_{\ge 1}, \gamma_0 \le 1$, and $\delta \in (0, 1]$ such that $B - 1 \le \min \{\gamma_0 k, (1 + \delta)^{-1}k \}$ and for all $t \ge 0$ and all $x \in \{1, ..., B - 1\}$ with $\Pr[X_t = x] > 0$ we have the binomial condition $\left(X_{t+1} \mid X_t = x \right) \succeq \Bin(k, (1+\delta)x/k )$. Let $T = \min\{t \ge 0 \mid X_t \ge B\}$. Then $E[T] = O\left( \frac{\log B}{\delta} \right )$.
\end{theorem}

The following lemma provides a good bound on never having a successful mutation within~$N$ tries.

\begin{lemma}[{\cite[Corollary~$1.4.3$]{DoerrN20}}]
  \label{thm:boundOnNoSuccess}
  For all $x \in [0, 1]$ and $y \in \R_{> 0}$, it holds that $(1 - x)^y \leq \frac{1}{1 + xy}$.
\end{lemma}

\begin{proof}[Proof of \Cref{lem:numberOfElementsBoundOMMM}]
  In \omm, for all $x, y \in \{0,1\}^n$, we have $x \nprec y$ and $y \nprec x$.
  Thus, all individuals in $R_t$ have rank $1$ in the non-dominated sorting of $R_t$.
  Therefore the critical rank is $1$, so $F_{i^*} = R_t$.
  The result follows by applying \Cref{lem:numberOfElementsBound} with $m = 2$, $C = N$, and $S = n + 1$, noting that $\card{f(R_t)} \le n + 1$.
\end{proof}

\begin{proof}[Proof of \Cref{lem:discoverNewElement}]
  Consider we are in iteration $t_0$ when we have $x_0$ in the parent population $P_{t_0}$.
  Let us denote by $X^{v.i}_t$ the number of individuals $x~ \in~ R_{t+t_0}$ with $f_i(x) = v$, that is $X^{v.i}_t \coloneqq \card{\{x \in R_{t+t_0} \mid f_i(x) = v\}}$.
  Let $B_v \coloneqq ~ \max \left(1, \min \left( {\lfloor \frac{N}{n+1} \rfloor - 4}, \lceil \frac{n}{v} \rceil \right) \right)$.
  Hence, $B_v \ge 1$.
  Denote by $G^i_v$ the number of iterations needed to reach the $B_v$ bound, that is $G^i_v \coloneqq \min \{t \in \N \mid X^{v.i}_t \ge B_v \}$.
  Also denote by $D^i_v$ the number of iterations needed to discover an individual $y$ with $f_i(y) = v - 1$ when we have at least $B_v$ individuals $x$ with $f_i(x) = v$.

  We see that $T^i_v \le G^i_v + D^i_v$, because obtaining an individual $x_f$ with $f_i(x_f) = v - 1$ requires at most as many iterations as obtaining such individual after reaching the $B_v$ bound for $X_t^{v.i}$. Hence, by the linearity of expectation as well as the law of total expectation,
  \begin{equation}
    E[T^i_v] \le E[G^i_v] + E[D^i_v] = E[G^i_v] + E\left[ E[D^i_v \mid G^i_v] \right].
    \label{eq:TEquation}
  \end{equation}

  We now derive bounds for $E[G^i_v]$ and $E[D^i_v]$.

  For $E[G^i_v]$, we know $X^{v.i}_0 \ge 1$, as $x_0 \in P_{t_0 + 0}$.
  We rewrite $B_v = \min \left(  \max\left(1, { \lfloor \frac{N}{n+1} \rfloor - 4} \right),\max\left(1, \lceil \frac{n}{v} \rceil \right) \right)$ and notice that $\lceil \frac{n}{v} \rceil \ge 1$ as $v \le n$.
  Therefore we get that $B_v = \min \left(  \max\left(1, { \lfloor \frac{N}{n+1} \rfloor - 4} \right), \lceil \frac{n}{v} \rceil \right)$.
  We also know from \Cref{lem:keepParetoFrontOMMM,lem:numberOfElementsBoundOMMM} that, since $B_v \le \max \left(1, \lfloor \frac{N}{n + 1} \rfloor - 4 \right)$, for any iteration $t < G^i_v$, at least $\min(X_t, B_v)$ individuals with $f_i(x) = v$ are selected in $P_{t_0 + t + 1}$ and consequently kept in $R_{t_0 + t + 1}$.
  However, as $t < G^i_v$, it means that $X_t < B_v$ and thus $\min(X_t, B_v) = X_t$, which means that at least $X_t$ individuals with $f_i(x) = v$ are kept in $R_{t_0 + t + 1}$.
  In addition to those, the algorithm produces some new individuals by mutation.
  Since we generate offspring by $N$ times choosing the parent uniformly at random and then applying standard bit mutation, the probability of picking an individual $x$ with $f_i(x) = v$ as parent is $\frac{X^{v.i}_t}{N}$ by definition. Hence, the probability of selecting such an individual and creating a copy of it as an offspring is $\frac{X^{v.i}_t}{N} \left( 1 - \frac{1}{n} \right)^n \ge \frac{X^{v.i}_t}{eN} \left(1 - \frac{1}{n} \right)   \ge \frac{X^{v.i}_t}{2eN}$ as we assume $n$ big enough, hence $n \ge 2$.

  We aim to apply \Cref{th:multiplicativeUpDriftTheorem} to $(X_t)_{t \in \N}$ and to~$G^i_v$.
  For all $t < G^i_v$, we have $X^{v.i}_{t+1} \ge~ X^{v.i}_{t} + \Bin \left( N, \frac{X^{v.i}_t}{2eN} \right)$, as we repeat $N$ times the selection of parents.
  We apply \Cref{th:multiplicativeUpDriftTheorem} with $B = B_v$, $\delta = \frac{1}{2e}$, and $k = N$, noting that $B_v - 1 \le  \lfloor \frac{N}{n + 1} \rfloor - 4 \le \frac{N}{2} \le \min\{N, (1 + \delta)^{-1}N\}$ and we get that $E[G^i_v] = O\left(\frac{\log B_v}{e}\right) = O\left(\log B_v\right)$.
  Since $B_v \le \lceil \frac{n}{v} \rceil$ by how we rewrote $B_v$, we get that
  \begin{equation}
    E[G^i_v] = O\left(\log \left\lceil \frac{n}{v} \right\rceil \right).
    \label{eq:G^i_v}
  \end{equation}

  For $E[D^i_v]$, we know $E[D^i_v] = E[E[D^i_v \mid G^i_v]]$ by the law of total expectation.
  This condition assures we start with at least $B_v$ individuals $x$ with $f_i(x) = v$.
  Since $B_v \le \max \left(1, \lfloor \frac{N}{n + 1} \rfloor - 4 \right)$, by \Cref{lem:keepParetoFrontOMMM,lem:numberOfElementsBoundOMMM}, we know that for all following iterations, we never have less than $B_v \ge 1$ individuals $x$ with $f_i(x) = v$.

  The probability to select a parent $x$ with $f_i(x) = v$ and to apply mutation to get an individual $y$ with $f_i(y) = v - 1$ is greater than $\frac{B_v}{N} \cdot v \cdot \frac{1}{n} \left( 1 - \frac{1}{n} \right)^{n-1} \ge \frac{B_v}{N} \frac{v}{en}$, as this only considers the case of flipping one of the $v$ bits.
  Since we apply this process $N$ times, the probability that in a single offspring generation we discover at least one such $y$ with $f_i(y) = v - 1$ is greater than $1 - \left( 1 - \frac{B_v}{N} \frac{v}{en} \right)^N$.
  By \Cref{thm:boundOnNoSuccess}, $\left( 1 - \frac{B_v}{N} \frac{v}{en} \right)^N \le \frac{1}{1 + \frac{B_v v}{en}} = \frac{en}{en + B_v v}$.
  Hence we have that the probability to discover $y$ is greater than $1 - \frac{en}{en + B_v v} = \frac{B_v v}{en + B_v v}$.

  Thus, $E \left[ D^i_v \mid G^i_v \right] \le \frac{en + v B_v}{v B_v} = \frac{en}{v B_v} + 1$. Furthermore, since $B_v = \min \left(  \max\left(1, { \lfloor \frac{N}{n+1} \rfloor - 4} \right), \lceil \frac{n}{v} \rceil \right)$ we have that
  \begin{align}
    E\left[D^i_v\right]
     & = E\left[E\left[D^i_v \mid G^i_v\right]\right]
    \le E\left[\frac{en}{v B_v} + 1\right] = \frac{en}{v B_v} + 1 \nonumber                                                                                                                  \\
     & = O\left(\frac{n}{v B_v} + 1\right) \nonumber                                                                                                                                         \\
     & = O\left( \frac{n}{v} \frac{1}{\min \left(  \max\left(1, { \lfloor \frac{N}{n+1} \rfloor - 4} \right), \lceil \frac{n}{v} \rceil \right)} + 1\right) \nonumber                        \\
     & = O \left( \frac{n}{v} \max \left({\frac{1}{\max\left(1, { \lfloor \frac{N}{n+1} \rfloor - 4} \right)}}, {\frac{1} {\lceil \frac{n}{v} \rceil}}\right)  + 1 \right) \label{eq:Dline3} \\
     & = O\left( \max\left(\frac{n ^2}{N v},{ \frac{n v} {v n}}\right) + 1 \right)
    =  O\left( \frac{n^2}{Nv} + 1\right).
    \label{eq:D^i_v}
  \end{align}

  Note that since $N \ge 4(n + 1)$ we have that $\max\left(1, { \lfloor \frac{N}{n+1} \rfloor - 4} \right) = \Theta(\frac{N}{n}) $, which we substitute into \cref{eq:Dline3}.
  By \cref{eq:TEquation} and adding up the bounds from \cref{eq:G^i_v}, \cref{eq:D^i_v} we get $E[T_v^i]~= ~O \left(\log \lceil \frac{n}{v} \rceil + \frac{n^2}{Nv} + 1 \right)$.
\end{proof}

\begin{proof}[Proof of \Cref{th:runtimeGeneralBalancedNSGAOneMinMax}]
  We use the notation from \Cref{lem:discoverNewElement}.
  The expected number of iterations to obtain individuals containing objective values $(k_0, n-~k_0)$, $(k_0-~1, n~-~k_0+1)$, $(k_0-2,n-k_0+2)$,$\ldots$, $(\alpha, n-\alpha)$ is at most $\sum_{v=1 + \alpha}^{k_0} E[T^{(1)}_v]$.
  Similarly, the expected number of iterations to obtain individuals containing objective values $(k_0, n - k_0)$, $(k_0 +1, n - k_0 -~ 1)$, $(k_0 + 2, n - k_0 - 2)$,$\ldots$,$(n - \alpha, \alpha)$ is at most $\sum_{v=1 + \alpha}^{n - k_0} E[T^{(2)}_v]$.
  Therefore, the expected number of iterations to cover the restriction of the Pareto front is at most $\sum_{v=1 + \alpha}^{v=k_0} E[T^{\left(1\right)}_v] + \sum_{v=1 + \alpha}^{n - k_0} E[T^{\left(2\right)}_v] \le \sum_{v=1 + \alpha}^{n - \alpha} \left(E[T^{\left(1\right)}_v] + E[T^{\left(2\right)}_v]\right)$.

  By \Cref{lem:discoverNewElement}, $E[T^{\left(i\right)}_v] = O\left(\log{\lceil \frac{n}{v} \rceil} + \frac{n^2}{Nv} + 1\right)$ for $i \in \{1, 2\}$.
  Therefore the number of iterations to cover the Pareto front restricted to individuals $x$ with~$\card{x}_1 \in [\alpha..n - \alpha]$ is less than
  \begin{align*}
     & \sum_{v=1 + \alpha}^{n - \alpha} O\left(\log{\left\lceil \frac{n}{v} \right\rceil} + \frac{n^2}{Nv} + 1\right) \\
     & = O\left(n +\sum_{v=1}^{n} \log{\frac{n + v}{v} } + \sum_{v=1}^{n} \frac{n^2}{Nv} \right)                      \\
     & = O\left(n + \log{\frac{(2n)!}{(n!)(n!)}} + \frac{n^2}{N} \sum_{v=1}^{n} \frac{1}{v} \right)                   \\
     & = O\left(n + \log{4^n} + \frac{n^2}{N} \log{n}\right)                                                          \\
     & = O\left(n + \frac{n^2}{N} \log{n} \right).
  \end{align*}
  We have used the fact that the central coefficient $\frac{(2n)!}{(n!)^2} \le 4^n$.
\end{proof}

\subsection*{Runtime Analysis on Bi-Objective \ojzj}

The following lemma shows that non-dominated objective values in the combined parent and offspring population survive into the next generation when the population size is sufficiently large. The proof is identical to the proof of the corresponding result for the classic \nsga \cite[Lemma~$1$]{DoerrQ23tec} since both algorithms in the considered setting let all individuals with positive crowding distance survive (hence our new tie-breaker is not relevant for these individuals).
\begin{lemma}
  \label{lem:keepParetoFrontOJZJ}
  Consider one iteration of the balanced \nsga with population size $N \ge 4(n - 2k + 3)$ optimizing \ojzj.
  If in some iteration $t \in \N$ the combined parent and offspring population $R_t$ contains an individual $x$ of rank $1$, then the next parent population $P_{t+1}$ contains an individual $y$ such that $f (y) = f (x)$.
  Moreover, if an objective value on the Pareto front appears in $R_t$, it is kept in all future iterations.
\end{lemma}

We consider the following three stages.
\begin{itemize}
  \item \emph{Stage 1}: $P_t \cap S^*_{\mathrm{In}} = \emptyset$. The algorithm aims to find the first individual with objective value in $F^*_{\mathrm{In}}$.
  \item \emph{Stage 2}: There exists a $v \in F^*_{\mathrm{In}}$ such that $v \notin f(P_t)$. The algorithm tries to cover the entire set $F^*_{\mathrm{In}}$.
  \item \emph{Stage 3}: $F^*_{\mathrm{In}} \subseteq f(P_t)$, but $F^*_{\mathrm{Out}} \not\subseteq f(P_t)$. The algorithm aims to find the extremal values of the Pareto front, $F^*_{\mathrm{Out}}$.
\end{itemize}

\begin{proof}[Proof of \Cref{lem:stageOne}]
  Assume that the initial population contains no individual in the inner part of the Pareto front (as otherwise stage 1 takes no iterations). Hence all individuals have less than $k$ zeros or less than $k$ ones. Let $x$ be such an individual, and assume with out loss of generality that it has $\ell < k$ ones. The probability that one application of the mutation operator generates a search point in the inner part of the Pareto front is at least
  \begin{align*}
    \binom{n-\ell}{k-\ell} n^{-(k-\ell)}\left(1-\frac 1n\right)^{n - (k-\ell)} & \ge \binom{n}{k}n^{-k}\frac 1e  \\
                                                                               & \ge \frac{1}{ek^k} \eqqcolon p,
  \end{align*}
  where we used the estimate $\binom{n}{k} \ge (\frac nk)^k$. Note that this $p$ is a lower bound for the event that one offspring generation creates an individual in the inner part of the Pareto front. Hence it takes an expected number of $1/p$ such attempts. These happen in iterations of exactly $N$ attempts, proving the claimed bound.
\end{proof}

\begin{proof}[Proof of \Cref{lem:stageThree}]
  Consider one iteration $t$ of stage $3$.
  Let $B_k \coloneqq \max\left(1, \lfloor\frac{N}{n - 2k + 3}\rfloor - 4\right)$.

  We know that there is an $x \in P_t$ such that $\card{x}_1 = k$.
  We also know that the inner Pareto front $F^*_I$ is already covered and we only need to discover the extremal values $F^*_O$. Hence, the critical rank of the non-dominated sorting is~$1$.
  Therefore, by \Cref{lem:numberOfElementsBound} with $m = 2$ and $C = N$, we get that as long as we have less than $\frac{N}{\card{f(F_{i^*})}} - 4$ elements with a certain objective value, we keep all of them.
  Using \Cref{lem:keepParetoFrontOJZJ} to make sure there is at least $1$ individual kept and the fact that $B_k \le \max\left(1, \frac{N}{\card{F^*}} - 4\right) \le \max \left(1,  \frac{N}{\card{f(F_{i^*})}} - 4 \right)$, we reuse the arguments of \cref{eq:G^i_v} from \Cref{lem:discoverNewElement} to get the number of iterations needed to reach $B_k$ individuals $x$ with $\card{x}_1 = k$.
  Hence, the expected number of iterations to reach $B_k$ is $O(\log B_k) = O(\log \frac{N}{n})$.

  Once we have $B_k$ individuals $x$ with $\card{x}_1 = k$, we compute the probability to discover $0^n$.
  Since we are considering $N$ times uniformly selecting a parent, the probability to select an individual with $\card{x}_1 = k$ is $\frac{B_k}{N} = \frac{1}{N} \max\left(1, \lfloor\frac{N}{n - 2k + 3}\rfloor - 4\right) \ge \max \left(\frac{1}{N}, \frac{1}{n - 2k + 3} - \frac{5}{N} \right) \ge \frac{1}{6(n - 2k + 3)}$, where the last inequality comes from the fact that, if $4(n - 2k + 3) \le N \le 6(n - 2k + 3)$ then $\frac{1}{N} \ge \frac{1}{6(n - 2k + 3)}$, and if $N \ge 6(n - 2k + 3)$ then $\frac{1}{n - 2k + 3} - \frac{5}{N} \ge \frac{1}{n - 2k + 3} - \frac{5}{6(n - 2k + 3)} = \frac{1}{6(n - 2k + 3)}$.
  Conditioning on an $x$ with $\card{x}_1 = k$ being selected, the probability of generating $0^n$ is $\frac{1}{n}^k \left(1 - \frac{1}{n} \right)^{n - k} \ge \frac{1}{en^k}$.
  Therefore, the probability of generating $0^n$ from one parent selection process is greater than $\frac{1}{6en^{k}(n- 2k + 3)} $.
  Since we repeat this process $N$ times with uniformly random parents, the probability of generating $0^n$ in this iteration is at least $1 - \left(1 - \frac{1}{6en^{k}(n - 2k+3)}\right)^N \ge \frac{N}{N + 6en^{k}(n - 2k+3)} $, by \Cref{thm:boundOnNoSuccess}.
  Hence, the expected number of iterations needed to generate $0^n$ after reaching $B_k$ individuals is less than $1 + \frac{6en^{k}(n-2k+3)}{N}$.

  Therefore, the expected number of iterations needed to generate $0^n$ is $O(\log \frac{N}{n}) + 1 + \frac{6en^{k}(n-2k+3)}{N}$.
  The case for $1^n$ is symmetrical, which concludes the proof.
\end{proof}

\begin{proof}[Proof of \Cref{th:runtimeBalancedNSGAOJZJ}]
  The sum of the expected durations of Stage 1 to 3 is $O\left(n + \log \frac{N}{n} + \frac{n^{k+1}}{N}\right)$ iterations. We note that $\log \frac Nn$ is $O(n)$ for $N \le 4^n$, hence in this case the theorem is already proven.

  Hence assume that $N > 4^n$. Then the probability that the initial population does not contain a particular search point $x \in \{0,1\}^n$ is $(1 - 2^{-n})^N \le \exp(-2^{-n}N) \le \exp(-N^{1/2})$. By a union bound, the probability that there is an $x$ missing in the initial population, is at most $2^n \exp(-N^{1/2})$. Only in this case, the runtime of the algorithm is a positive number of iterations. Since our proofs above have not made any particular assumptions on the initial population, the $O\left(n + \log \frac{N}{n} + \frac{n^{k+1}}{N}\right)$ iterations bound is valid also conditional on this initial situation. Hence for $N > 4^n$, the expected number of iterations to find the full Pareto front is at most $2^n \exp(-N^{1/2}) O\left(n + \log \frac{N}{n} + \frac{n^{k+1}}{N}\right) = o(1)$ iterations.
\end{proof}

\subsection*{Runtime Analysis on Bi-Objective \lotz}

\begin{lemma}[Lemma~$7$ from \cite{ZhengD23aij}]
  Consider one iteration of the balanced \nsga with population size $N \geq 4(n + 1)$ optimizing the \lotz function.
  Assume that in some iteration $t \in \N$, the combined parent and offspring population~$R_t$ contains a solution~$x$ with rank~$1$.
  Then also the next parent population~$P_{t + 1}$ contains an individual~$y$ with $f(y) = f(x)$.
  In particular, for all $k \in [0 .. n]$, once the parent population contains an individual with objective value $(k, n - k)$, it does so for all future iterations.
\end{lemma}

\begin{proof}[Proof of \Cref{lem:lotzQuicklyIncreasingPopulation}]
  For all $t \in \N$, let $X_t = |Y_t|$.
  We assume that $B > 1$, as the theorem follows directly otherwise, since $X_0 \geq 1$.
  In addition, for the sake of simplicity, we assume implicitly for the remainder of the proof that we condition on~$t_0$.

  We aim to apply the multiplicative up-drift theorem (\Cref{th:multiplicativeUpDriftTheorem}) to $(X_t)_{t \in \N}$ and~$T$.
  We say that \emph{the algorithm finds an improvement} if and only if there is a $t^* \in \N$ and a $z \in R_{t^* + t_0}$ such that $f_1(z) > v$ and such that, for all $t \in [0 .. t^* - 1]$, it holds that $X_t < B$.
  Note that we stop once the algorithm finds an improvement.

  Let $\delta = \frac{1}{2e}$, and let $t \in \N$ with $t < T$.
  We show that for all $a \in [B - 1]$, it holds that $(X_{t + 1} \mid X_t = a) \succeq \Bin(N, (1 + \delta) a/N)$.
  Since $t < T$, the algorithm did not find an improvement.
  Hence, all individuals in~$Y_t$ have rank~$1$.
  Furthermore, it holds by \Cref{lem:lotzQuicklyIncreasingPopulation} that at least one individual from~$Y_t$ is selected for the next iteration.
  Similarly, \Cref{lem:numberOfElementsBound} guarantees that at least $\min(\frac{N}{n + 1} - 4, X_t)$ individuals get selected.
  Since $t < T$ and thus $X_t < B$, and since $B > 1$, it holds that $\min(\frac{N}{n + 1} - 4, X_t) = X_t$.
  Thus, all individuals in~$Y_t$ are selected for the next iteration.
  In addition, the algorithm produces~$N$ offspring, which may increase~$X_t$.

  When considering the process of generating a single offspring~$y$, it increases~$X_t$ if it holds that $f_1(y) = v$.
  The probability of creating such an improving~$y$ consists of (i) choosing an individual~$x \in Y_t$, and (ii) creating a copy of~$x$.
  The probability for (i) is at least $X_t/N$.
  And the probability for (ii) is at least (for sufficiently large~$n$) $(1 - \frac{1}{n})^n \geq \frac{1}{2e}$.
  Thus, the probability of creating an improving~$y$ is at least $\delta \frac{X_t}{N}$.
  Let $Z \sim \Bin(N, \delta \frac{X_t}{N})$.
  Since the algorithm generates~$N$ offspring, it follows that $X_{t + 1} \geq X_t + Z$, thus showing $(X_{t + 1} \mid X_t = a) \succeq \Bin(N, (1 + \delta) a/N)$.

  In order to apply \Cref{th:multiplicativeUpDriftTheorem}, we note that $\delta \in (0, 1]$ and that $(1 + \delta)^{-1} = \frac{2e}{2e + 1} \geq \frac{1}{2}$.
  Since $B > 1$, we obtain
  \begin{align*}
    B - 1
    \leq \left\lfloor \frac{N}{n + 1} \right\rfloor - 4
    \leq \frac{N}{2}
    \leq \min \{N, (1 + \delta)^{-1} N\} .
  \end{align*}
  Thus, the requirements for \Cref{th:multiplicativeUpDriftTheorem} are satisfied, and we obtain $E[T'] = O(\log B)$.
  This proves the first part of the statement.

  The second part with exchanging~$f_1$ by~$f_2$ follows by noting that individuals that have a maximal $f_2$-value also have rank~$1$.
  The other arguments remain unchanged, thus concluding the proof.
\end{proof}

\begin{proof}[Proof of \Cref{thm:runtimeBalancedNSGALOTZ}]
  We prove the theorem by considering two phases.
  While differing slightly in details, the two phases follow mainly the same outline.
  Thus, we only provide one proof, utilizing variables to mark the different cases.

  Phase~$1$ considers the time until the algorithm's population contains for the first time the all-$1$s bit string.
  Let~$T_1$ denote this time.

  Phase~$2$ starts at $T_1 + 1$ and denotes the remaining time until the algorithm covers the entire Pareto front.
  For this phase, we require pessimistically that we start at objective value $(n, 0)$ and need to create consecutively individuals whose first component in objective value is smaller, not skipping any objective value.
  Let~$T_2$ denote the first time under these assumptions that an individual with objective value $(0, n)$ is in the population.

  The runtime of the algorithm on \lotz is bounded from above by $T_1 + T_2$.
  Hence, its expected runtime is bounded from above by $E[T_1] + E[T_2]$.
  We bound either expected value in the same way, relying on variables that differ between the phases.

  Consider either phase.
  Let $t \in \N$, and let $v_t \in [0 .. n]$ denote the current progress of a phase, which means the following:
  For phase~$1$, the quantity~$v_t$ is $\max_{y \in R_t} f_1(y)$.
  For phase~$2$, the quantity~$v_t$ is $\min \{i \in [0 .. n] \mid \forall j \in [i .. n] \exists y \in R_t\colon f(y) = (j, n - j)\}$.
  We use the terminology of making progress, wich means the following for each phase:
  For phase~$1$, we make progress when $v_{t + 1} > v_t$.
  For phase~$2$, we make progress when $v_{t + 1}< v_t$.
  Let $p \in (0, 1]$ denote a lower bound on the the probability of the following event when considering a single iteration of creating offspring:
  Create an offspring~$y$ that makes progress, conditional on choosing an $x \in P_t$ with $f_1(x) = v_t$.

  We split the current phase into segments, where each segment denotes the time until we make progress.
  Note that either phase has at most~$n$ segments.
  For each segment, we first wait until the current population~$R_t$ has at least $\max(1, \lfloor \frac{N}{4(n + 1)} \rfloor - 4) \eqqcolon B$ individuals with $f_1(v_t)$.
  By \Cref{lem:lotzQuicklyIncreasingPopulation}, the expected duration of this step is $O(\log B) = O(\log \frac{N}{n + 1})$.
  Furthermore, by \Cref{lem:numberOfElementsBound,lem:lotzQuicklyIncreasingPopulation}, at least~$B$ such individuals are maintained in~$P_t$ from now on.

  Afterward, we wait until we make progress by waiting for the following event to occur:
  During mutation, we choose an individual $x \in P_t$ with $f_1(x) = v_t$, and the created offspring makes progress.
  Due to the previous paragraph, the probability of choosing an individual~$x$ as intended is at least~$\frac{B}{N}$.
  By definition, the probability of then creating an offspring that makes progress is at least~$p$.
  Hence, the probability of making progress with a single mutation is at least $\frac{B}{N} p \eqqcolon q$.

  Since the algorithm creates~$N$ offspring in each iteration, we estimate the probability that at least one of the offspring makes progress as follows.
  To this end, we apply \Cref{thm:boundOnNoSuccess} and obtain a lower bound of
  \begin{align*}
    1 - (1 - q)^N
    \geq 1 - \frac{1}{1 + qN}
    = \frac{qN}{1 + qN} .
  \end{align*}
  Hence, we wait at most $\frac{1 + qN}{qN} = 1 + \frac{1}{qN} = 1 + \frac{1}{Bp}$ in expectation until we make progress and end a segment.

  In total, each segment lasts in expectation at most $O(\log \frac{N}{n + 1}) + 1 + \frac{1}{Bp} = O(\frac{n}{Np} + \log \frac{N}{n + 1})$.
  Since each phase has at most~$n$ segments, each phase lasts at most $O(\frac{n^2}{Np} + n \log \frac{N}{n + 1})$.

  Last, we determine a lower bound for~$p$ for both phases.
  We start with phase~$1$.
  Given an individual~$x$ as stated in the definition of~$p$, an offspring makes progress if it flips the leftmost~$0$ in~$x$ and copies the rest.
  That is, $p \geq \frac{1}{n} (1 - \frac{1}{n})^{n - 1} \geq \frac{1}{e n}$.
  For phase~$2$, we get the same bound, with the difference that we flip the rightmost~$1$ (which still has probability~$\frac{1}{n}$).
  Substituting~$p$ into our bound concludes the runtime part of the proof.

  The number of function evaluations follows from noting that each iteration requires an evaluation of~$N$ objective values.
\end{proof}

\subsection*{Empirical Runtime Analysis}

\Cref{fig:lotzJointPlot,fig:ojzjJointPlot} showcase our results explained in the empirical section for the functions \lotz and \ojzj with $k = 3$, respectively.
Now, $M$ denotes the size of the Pareto front of the respective problem.
We see that the results are qualitatively similar to \omm.
A classic \nsga is less robust to a suboptimal choice of~$N$ than the balanced \nsga.

\Cref{tab:significances} lists the statistical significances of the balanced \nsga having a faster runtime than the classic \nsga.
We see that the significance typically increases when increasing~$N$.
Especially, when choosing $N \in \{8M, 16M\}$ (where possible), the results are always statistically significant, even for a $p$-value of just~$0.001$.
This suggests that there is a clear advantage in the runtime of the balanced \nsga over the classic one when choosing suboptimal values of~$N$.

\begin{figure}[t]
  \includegraphics[width = \columnwidth]{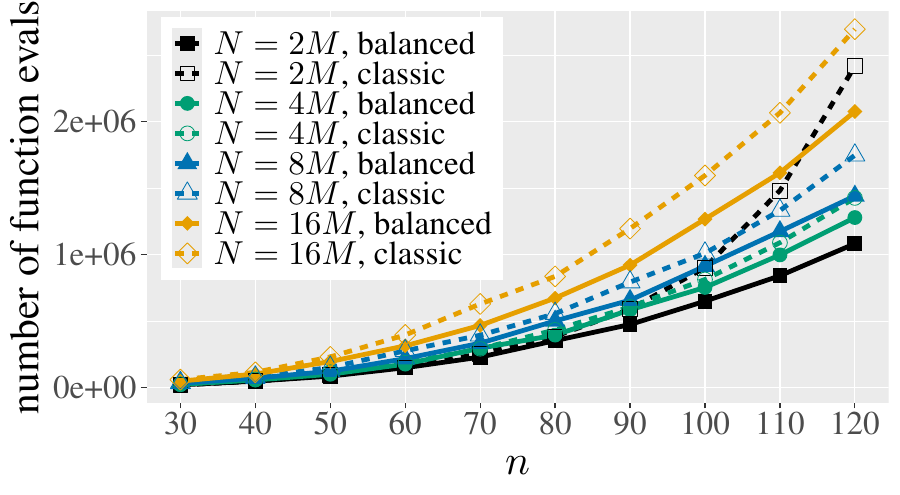}
  \caption{
    The average number of function evaluations of the classic (dashed lines) and the balanced (solid lines) \nsga optimizing \lotzLong, for the shown population sizes~$N$ and problem sizes~$n$.
    The value~$M$ denotes the size of the Pareto front, i.e., $M = n + 1$.
    Each point is the average of~$50$ independent runs.
  }
  \label{fig:lotzJointPlot}
\end{figure}

\begin{figure}
  \includegraphics[width = \columnwidth]{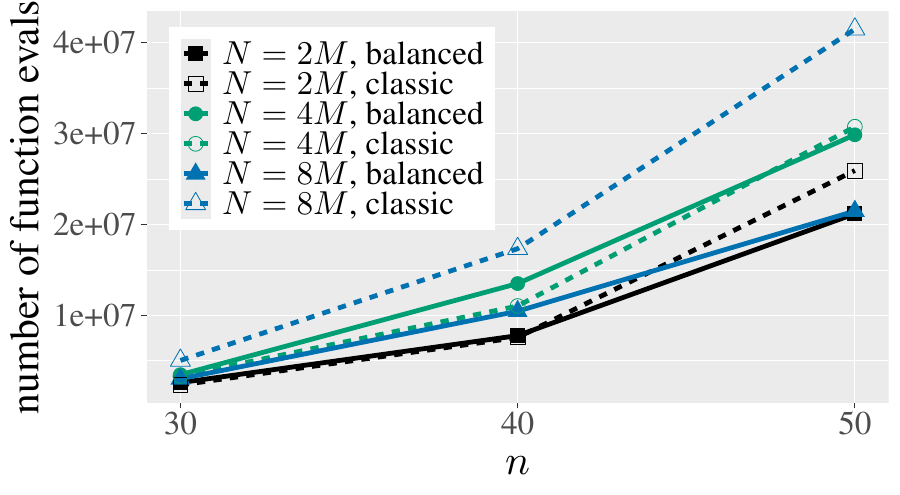}
  \caption{
    The average number of function evaluations of the classic (dashed lines) and the balanced (solid lines) \nsga optimizing \ojzjLong with $k = 3$, for the shown population sizes~$N$ and problem sizes~$n$.
    The value~$M$ denotes the size of the Pareto front, i.e., $M = n - 2k + 3 = n - 3$.
    Each point is the average of~$50$ independent runs.
  }
  \label{fig:ojzjJointPlot}
\end{figure}

\begin{table*}[b]
  \centering
  \caption{
    The statistical significance of the number of function evaluations per value of~$n$ of the balanced \nsga being lower that that of the classic \nsga, for the data depicted in \Cref{fig:ommJointPlot,fig:lotzJointPlot,fig:ojzjJointPlot}, split with respect to the choice of the population size~$N$.
    The value~$M$ denotes the size of the Pareto front, as reported in the respective figures.
    The $p$-value is the result of a one-sided Mann--Whitney~$U$ test, rounded to three digits, indicating values that are less than~$0.001$.
    Bold $p$-values indicate those that are at most~$0.05$, which are considered to be statistically significant.
  }
  \label{tab:significances}
  \begin{tabular}{*{3}{lrrr@{\hspace*{2 em}}}}
    benchmark & $n$   & $N$   & $p$-value              & benchmark & $n$   & $N$   & $p$-value              & benchmark & $n$  & $N$  & $p$-value              \\
    \cmidrule[0.75 pt](r{1.5 em}){1-4}
    \cmidrule[0.75 pt](r{1.5 em}){5-8}
    \cmidrule[0.75 pt](r{1.5 em}){9-12}
    \omm      & $30$  & $2M$  & $0.309$                & \lotz     & $30$  & $2M$  & $0.859$                & \ojzj     & $30$ & $2M$ & $0.61$                 \\
              &       & $4M$  & $\boldsymbol{0.001}$   &           &       & $4M$  & $0.274$                &           &      & $4M$ & $0.62$                 \\
              &       & $8M$  & $\boldsymbol{< 0.001}$ &           &       & $8M$  & $\boldsymbol{0.007}$   &           &      & $8M$ & $\boldsymbol{< 0.001}$ \\
              &       & $16M$ & $\boldsymbol{< 0.001}$ &           &       & $16M$ & $\boldsymbol{< 0.001}$                                                    \\
              & $40$  & $2M$  & $0.119$                &           & $40$  & $2M$  & $0.615$                &           & $40$ & $2M$ & $0.799$                \\
              &       & $4M$  & $\boldsymbol{0.022}$   &           &       & $4M$  & $0.068$                &           &      & $4M$ & $0.938$                \\
              &       & $8M$  & $\boldsymbol{< 0.001}$ &           &       & $8M$  & $0.059$                &           &      & $8M$ & $\boldsymbol{< 0.001}$ \\
              &       & $16M$ & $\boldsymbol{< 0.001}$ &           &       & $16M$ & $\boldsymbol{0.025}$                                                      \\
              & $50$  & $2M$  & $0.077$                &           & $50$  & $2M$  & $0.086$                &           & $50$ & $2M$ & $\boldsymbol{0.034}$   \\
              &       & $4M$  & $0.242$                &           &       & $4M$  & $\boldsymbol{0.006}$   &           &      & $4M$ & $0.192$                \\
              &       & $8M$  & $\boldsymbol{< 0.001}$ &           &       & $8M$  & $\boldsymbol{< 0.001}$ &           &      & $8M$ & $\boldsymbol{< 0.001}$ \\
              &       & $16M$ & $\boldsymbol{< 0.001}$ &           &       & $16M$ & $\boldsymbol{< 0.001}$                                                    \\
              & $60$  & $2M$  & $\boldsymbol{0.01}$    &           & $60$  & $2M$  & $0.526$                                                                   \\
              &       & $4M$  & $0.486$                &           &       & $4M$  & $0.14$                                                                    \\
              &       & $8M$  & $\boldsymbol{< 0.001}$ &           &       & $8M$  & $\boldsymbol{< 0.001}$                                                    \\
              &       & $16M$ & $\boldsymbol{< 0.001}$ &           &       & $16M$ & $\boldsymbol{< 0.001}$                                                    \\
              & $70$  & $2M$  & $\boldsymbol{0.035}$   &           & $70$  & $2M$  & $\boldsymbol{0.031}$                                                      \\
              &       & $4M$  & $\boldsymbol{0.001}$   &           &       & $4M$  & $0.479$                                                                   \\
              &       & $8M$  & $\boldsymbol{< 0.001}$ &           &       & $8M$  & $\boldsymbol{< 0.001}$                                                    \\
              &       & $16M$ & $\boldsymbol{< 0.001}$ &           &       & $16M$ & $\boldsymbol{< 0.001}$                                                    \\
              & $80$  & $2M$  & $0.387$                &           & $80$  & $2M$  & $\boldsymbol{0.035}$                                                      \\
              &       & $4M$  & $0.179$                &           &       & $4M$  & $\boldsymbol{0.027}$                                                      \\
              &       & $8M$  & $\boldsymbol{< 0.001}$ &           &       & $8M$  & $\boldsymbol{0.004}$                                                      \\
              &       & $16M$ & $\boldsymbol{< 0.001}$ &           &       & $16M$ & $\boldsymbol{< 0.001}$                                                    \\
              & $90$  & $2M$  & $0.147$                &           & $90$  & $2M$  & $\boldsymbol{< 0.001}$                                                    \\
              &       & $4M$  & $0.266$                &           &       & $4M$  & $0.251$                                                                   \\
              &       & $8M$  & $\boldsymbol{< 0.001}$ &           &       & $8M$  & $\boldsymbol{< 0.001}$                                                    \\
              &       & $16M$ & $\boldsymbol{< 0.001}$ &           &       & $16M$ & $\boldsymbol{< 0.001}$                                                    \\
              & $100$ & $2M$  & $0.149$                &           & $100$ & $2M$  & $\boldsymbol{< 0.001}$                                                    \\
              &       & $4M$  & $\boldsymbol{0.011}$   &           &       & $4M$  & $\boldsymbol{0.021}$                                                      \\
              &       & $8M$  & $\boldsymbol{< 0.001}$ &           &       & $8M$  & $\boldsymbol{0.003}$                                                      \\
              &       & $16M$ & $\boldsymbol{< 0.001}$ &           &       & $16M$ & $\boldsymbol{< 0.001}$                                                    \\
              & $110$ & $2M$  & $\boldsymbol{0.011}$   &           & $110$ & $2M$  & $\boldsymbol{< 0.001}$                                                    \\
              &       & $4M$  & $0.053$                &           &       & $4M$  & $\boldsymbol{0.019}$                                                      \\
              &       & $8M$  & $\boldsymbol{< 0.001}$ &           &       & $8M$  & $\boldsymbol{< 0.001}$                                                    \\
              &       & $16M$ & $\boldsymbol{< 0.001}$ &           &       & $16M$ & $\boldsymbol{< 0.001}$                                                    \\
              & $120$ & $2M$  & $0.365$                &           & $120$ & $2M$  & $\boldsymbol{< 0.001}$                                                    \\
              &       & $4M$  & $\boldsymbol{0.017}$   &           &       & $4M$  & $\boldsymbol{< 0.001}$                                                    \\
              &       & $8M$  & $\boldsymbol{< 0.001}$ &           &       & $8M$  & $\boldsymbol{< 0.001}$                                                    \\
              &       & $16M$ & $\boldsymbol{< 0.001}$ &           &       & $16M$ & $\boldsymbol{< 0.001}$                                                    \\
  \end{tabular}
\end{table*}

In Figure~\ref{fig:4ommplot}, we show the runtimes of the balanced \nsga on the 4-objective \omm problem. We recall that the classic \nsga could not optimize this benchmark in subexponential time~\cite{ZhengD24many}. These difficulties were also seen in experiments, where, for example, the \nsga with population size four times the Pareto front size~$M$ in the first 1000 iterations (that is, more than $1\,600\,000$ function evaluations) never covered more than $60$\,\% of the Pareto front of the 4-objective \omm problem with problem size $n=40$.

Our results in Figure~\ref{fig:4ommplot} show that the balanced \nsga does not have these difficulties. For example, for $n=40$ and $N=4M$, the balanced \nsga computed the full Pareto front within less than $147\,153$ function evaluations on average.

\begin{figure}[t]
  \includegraphics[width = \columnwidth]{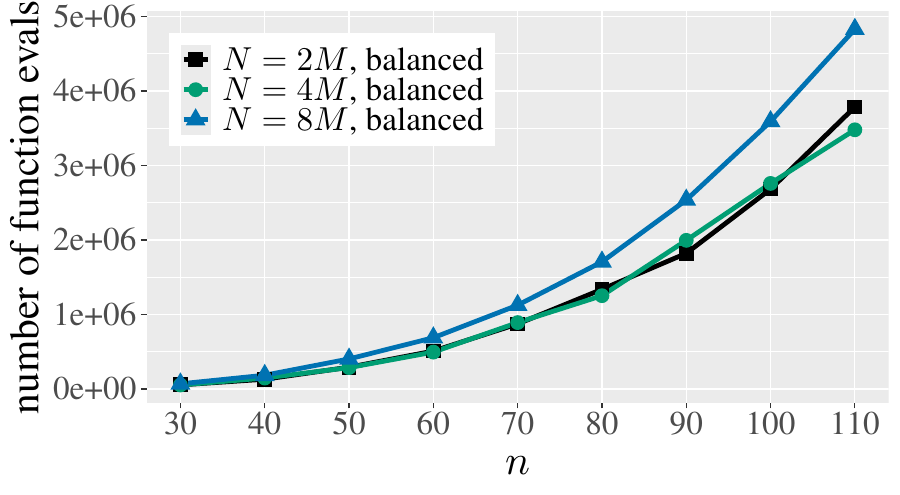}
  \caption{
    Average runtimes of the balanced \nsga with different population sizes $N$ on the 4-objective \omm problem.
    The value~$M$ denotes the size of the Pareto front, i.e., $M = (n/2 + 1)^2$.
    Each point is the average of~$50$ independent runs.
  }
  \label{fig:4ommplot}
\end{figure}

\end{document}